\PassOptionsToPackage{table}{xcolor}
\documentclass[]{fairmeta}

\usepackage[utf8]{inputenc}
\usepackage{url}
\usepackage{amsfonts}
\usepackage{nicefrac}
\usepackage{amsmath}
\usepackage{amssymb}
\usepackage{mathtools}
\usepackage{amsthm}
\usepackage{colortbl}
\usepackage{array}
\usepackage{wrapfig}
\usepackage{enumitem}

\definecolor{recgreen}{RGB}{220,245,220}
\newtcolorbox{recommendation}{colback=recgreen,colframe=recgreen!80!black,fonttitle=\bfseries,title=Takeaway,boxrule=0.5pt,arc=2pt,left=4pt,right=4pt,top=2pt,bottom=2pt}

\usepackage{algorithm}
\usepackage{algorithmic}

\newcommand{\mS}{\mathbf{m}_S}
\newcommand{\mF}{\mathbf{m}_F}

\theoremstyle{plain}
\newtheorem{theorem}{Theorem}[section]
\newtheorem{proposition}[theorem]{Proposition}

\theoremstyle{definition}

\theoremstyle{remark}

\title{Don't Let Gains FADE: Breaking Down Policy Gradient Weights in RL}

\author[1,2]{Juliette Decugis}
\author[1,*]{Sean O'Brien}
\author[2]{Francis Bach}
\author[1]{Gabriel Synnaeve}
\author[1]{Taco Cohen}

\affiliation[1]{FAIR at Meta}
\affiliation[2]{Inria, \'Ecole Normale Sup\'erieure}

\contribution[*]{Work done at Meta, now at the University of California, San Diego}

\abstract{Reinforcement learning post-training dramatically improves LLM reasoning, but suffers from training instability and diversity collapse. Advantage functions offer an appealing fix: they reshape the training objective, reweight which rollouts drive learning, and are trivial to implement. Yet a proliferation of methods makes it unclear which advantage to use and when. We focus on advantages of the form $A:=c(p)(r-b)$ where $p$ is a problem's average solve rate, $r$ its reward. We show $c$ controls the inference scaling laws of our model and $b$ the stability of training. To maximize both we introduce FADE (Focal Advantage with Dynamic Entropy), a self-adapting advantage that reads training dynamics to schedule the gradient weight automatically.}

\date{\today}
\correspondence{Juliette Decugis at \email{jdecugis@meta.com}}

\begin{document}

\maketitle

\section{Introduction}

Recent developments in reinforcement learning with verifiable rewards (RLVR) have unlocked rapid gains in language models' capabilities, especially in easily-verifiable domains like code generation and mathematics \citep{openai2024reasoning, Guo_2025,deepseek_grpo,drgrpo}. Although sparse rewards over long sequences make credit assignment challenging \citep{minsky1961steps,sutton1988learning,zhang2026reasoning}, pretrained LLMs provide strong behavioral priors \citep{gan2026neural,yan2025efficient} and fully resettable environments enable parallel rollout collection. These methods therefore follow a common recipe: sample multiple rollouts per problem, score them with a binary verifier, and update the policy via a weighted policy gradient \citep{Williams1992,schulman2015high}. Since we only have one reward for up to thousands of tokens, RL suffers from noisy gradients. This noise matters most for reasoning RL \citep{openai2024reasoning,Guo_2025}. We train on a fixed set of math and coding problems, not because those problems matter in themselves, but in the hope that the model acquires reasoning behaviours that transfer to tasks it never sees. What we want is breadth: many ways of approaching a problem, not one strategy fitted to the training set. 

Methods reduce this noise through the advantage $A$, the weight on each sample's gradient. The classic tool is a baseline: score each response against a reference, $A:=r-b$, so that the model only moves when a response does better or worse than $b$. Any $b$ that does not depend on the response keeps the gradient unbiased \citep{Williams1992}, so we are free to pick the one with the least variance $b^*$. \citet{greensmith2004variance,weaver2013optimal} showed:
\begin{equation}
b^*=\frac{\mathbb E\big[r\,\|s\|^2\big]}{\mathbb E\big[\|s\|^2\big]},
\qquad s=\nabla_\theta\log\pi_\theta(\tau\mid x),
\end{equation}
the average reward, weighted by how strongly each response moves the model. If successes and failures move it about equally, $b^*$ is simply $\mathbb E[r]$ estimated from the other rollouts of the same prompt \citep{ahmadian2024back}. In practice, however, choosing the best advantage is more complex for reasoning RL as it breaks three core assumptions:
\begin{itemize}[nosep,leftmargin=*]
  \item \textbf{Successes and failures look alike.} $b^*=\mathbb E[r]$ holds only if both carry the same score energy $\mathbb E\|s\|^2$ (Appendix~\ref{app:optimal_baseline}). In LLMs, failures are often longer and differ in entropy, motivating advantages that re-weight failures \citep{w_reinforce,arnal2026asymmetric}, shift or clip the baseline \citep{kim2026mc,CorPO}, or normalize per token \citep{yu2025dapo,drgrpo}.
  \item \textbf{The goal is single-attempt accuracy.} Harder problems need diverse samples, motivating pass@$k$ advantages \citep{tang2025optimizing,chen2025pass} and pass@$k$/pass@$1$ mixtures such as MaxRL \citep{tajwar2026maximumlikelihoodreinforcementlearning} and log-mean-exp \citep{lme}.
  \item \textbf{Unbiasedness is worth its variance.} GRPO's standard-deviation normalization \citep{deepseek_grpo} and related schemes \citep{andrychowicz2020matters} trade bias for lower variance.
\end{itemize}
As listed above, there have been many approaches to optimally weight trajectories in reasoning RL. Each claims different gains in sample efficiency, stability or diversity and comparing advantages is difficult because they differ along multiple axes simultaneously. Consider for example the pass@$8$ normalization \citep{tang2025optimizing}, which only upweights a correct rollout when it is the sole success in a batch of eight. This simultaneously shifts gradient mass toward hard problems, drops all negative gradient signal since incorrect rollouts receive zero weight, and reduces the overall gradient magnitude because most batches contain either zero or more than one success.

In this paper, we study what advantage changes are important for RL in the most complex setup: long reasoning traces. Based on the recent literature (Table~\ref{tab:weights}), we focus on advantages of the form $A:=c(p)\,(r-b)$ with $p$ the estimated average solve rate, $c(p)$ a prompt reweighting based on difficulty and $b$ the baseline. Through a unifying framework for comparing advantages (Section~\ref{sec:framework}) and large scale experiments with LLMs at the 7B to 35B scale we show:
\begin{itemize}[nosep,leftmargin=*]
  \item \textbf{Entropy is a feature of your advantage.} You can control the emphasis on sampled or unsampled tokens in your policy through the sign bias of your baseline. (Section~\ref{sec:entropy_collapse}).
  \item \textbf{The difficulty axis decides diversity.} The prompt weight $c$ shapes the $k$ vs. pass@$k$ scaling law, and focusing on solve rates $p\in[0.25,0.4]$ maximizes pass@$100$ (Section~\ref{sec:shape}).
  \item \textbf{FADE adapts both axes online.} It schedules the advantage from the solve rate and the entropy, and outperforms the best static advantages at every pass@$k$ (Section~\ref{sec:scheduled-advantages}).
\end{itemize}

\begin{figure*}[t]
    \centering
    \includegraphics[width=\linewidth]{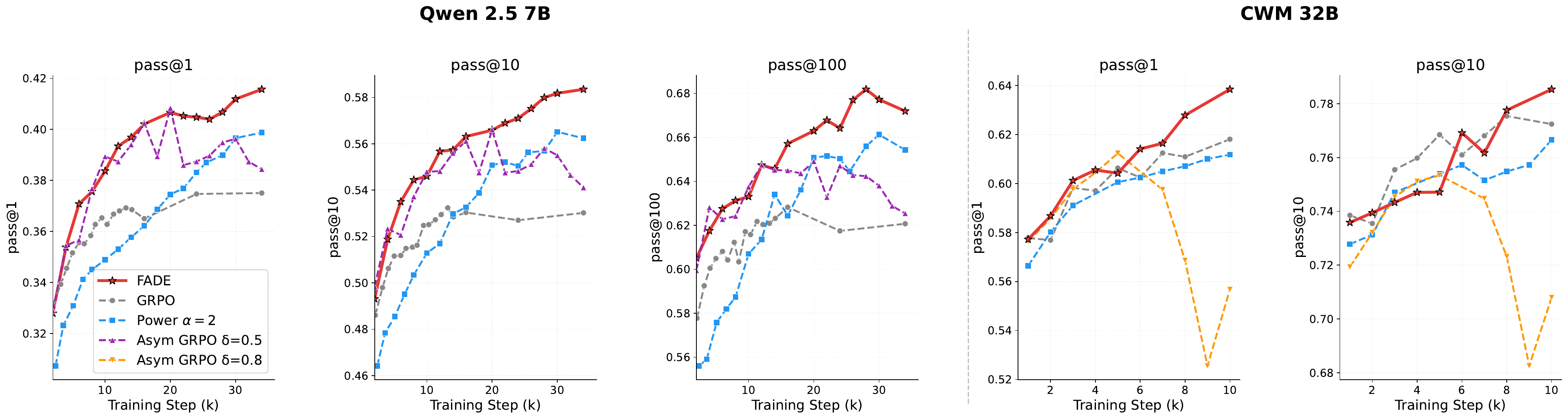}
    \caption{\textbf{FADE learns faster and better for all pass@k on LiveCodeBench v6} when compared to Dr.~GRPO and the best static advantages Power $\alpha$ and AsymGRPO with the best $\delta$ per model.}
    \label{fig:pass_at_k_scheduler}
\end{figure*}

\section{Advantages as Positive and Negative Gradient Mass}
\label{sec:framework}
We view the LLM as a policy $\pi_\theta$ that generates a trajectory $\tau=(a_1,\ldots,a_T)$ given a prompt $x$, and receives a scalar reward $r(\tau)$ at the end. The policy gradient~\citep{Williams1992} is the expected score weighted by the reward, $\nabla_\theta\mathbb{E}_{\pi_\theta}[r]=\mathbb{E}_{\pi_\theta}[r\,\nabla_\theta\log\pi_\theta(\tau)]$. Estimated from a few rollouts, it is noisy: a trajectory's score sums thousands of token scores, most of which have nothing to do with the reward. To reduce variance, practical methods replace $r$ by a weight $W(\tau)$, the advantage, often based on the per-prompt solve rate (Table~\ref{tab:weights}).

\subsection{Positive and Negative Gradient Mass}
\label{sec:pos_neg_mass}
We assume binary rewards, $r\in\{0,1\}$. Let $S$ and $F$ be the successful and failed trajectories, $p:=\sum_{\tau\in S}\pi_\theta(\tau)$ the solve rate and $q:=1-p$, which we use as a one-dimensional proxy for difficulty. An advantage gives each success a weight $w_S\ge0$ and each failure a weight $w_F\ge0$, both possibly depending on $p$:
\begin{equation}
W(\tau)=w_S\,\mathbb I[\tau\in S]-w_F\,\mathbb I[\tau\in F].
\label{eq:policy_weight}
\end{equation}
Following \citet{thrampoulidis2025advantage}, the update then splits into a push up on successes and a push down on failures:
\begin{equation}
g=\underbrace{w_S\,p}_{m_S}\,\underbrace{\mathbb E\big[\nabla_\theta\log\pi_\theta(\tau)\,\big|\,\tau\in S\big]}_{\nabla_S}-\underbrace{w_F\,q}_{m_F}\,\underbrace{\mathbb E\big[\nabla_\theta\log\pi_\theta(\tau)\,\big|\,\tau\in F\big]}_{\nabla_F}.
\label{eq:pos_neg_cond}
\end{equation}
We call $m_S$ and $m_F$ the positive and negative gradient mass. The default is sign balance, $m_S=m_F$, under which a group's advantages sum to zero as is the case for Dr.~GRPO~\citep{drgrpo} (Figure~\ref{fig:weigths_advantages}). However in practice sign biased advantages can help because we get different signal from successes and failures. A positive-bias can accelerate learning if we know failures are more noisy whereas a failure-bias can slow overfitting if we know successes are already collapsed (Section~\ref{sec:entropy_collapse}).

In practice we do not have access to the true $p$ and $q$, so we estimate them from $G$ Monte Carlo rollouts per prompt: $\bar p:=|\bar S|/G$ is the fraction of successful rollouts and $\bar q:=1-\bar p$.

\subsection{Two Knobs: Which Prompts and Which Sign}
\label{sec:knobs}
We study a specific subset of weights of the form $W(\tau) := c(\bar p)\,\big(r(\tau)-b(\bar p)\big)$. Setting $r=1$ for successes and $r=0$ for failures, the masses of Equation~\ref{eq:pos_neg_cond} are $m_S=c\,(1{-}b)\,p$ and $m_F=c\,b\,q$ so we can rewrite:
\begin{equation}
g:=\underbrace{c(p)}_{\text{which prompts}}\big[\underbrace{(1{-}b)\,p\,\nabla_S-b\,q\,\nabla_F}_{\text{which sign}}\big]
=c(p)\;p\,(1{-}p)\,\big(\nabla_S-\nabla_F\big)
\label{eq:pop_update}
\end{equation}

The second equality uses $p\,\nabla_S=-q\,\nabla_F$, which holds in the population limit because the score has zero mean, $\mathbb E[\nabla_\theta\log\pi_\theta(\tau)]=0$.

With $c(p)=1$ we ascend the mean solve rate $p$, and the per-prompt weight $p(1-p)$ peaks at $p=\tfrac12$. With $c(p)=1/p$ (MaxRL) we ascend $\log p$, which upweights hard prompts and forces the policy to focus on them. Using our framework, we can unify many recent methods as different forms of weighting policy gradients by either tilting the balance towards easier or harder prompts or positive or negative gradients as shown in Table~\ref{tab:weights}.

\begin{figure*}
    \centering
    \includegraphics[width=\linewidth]{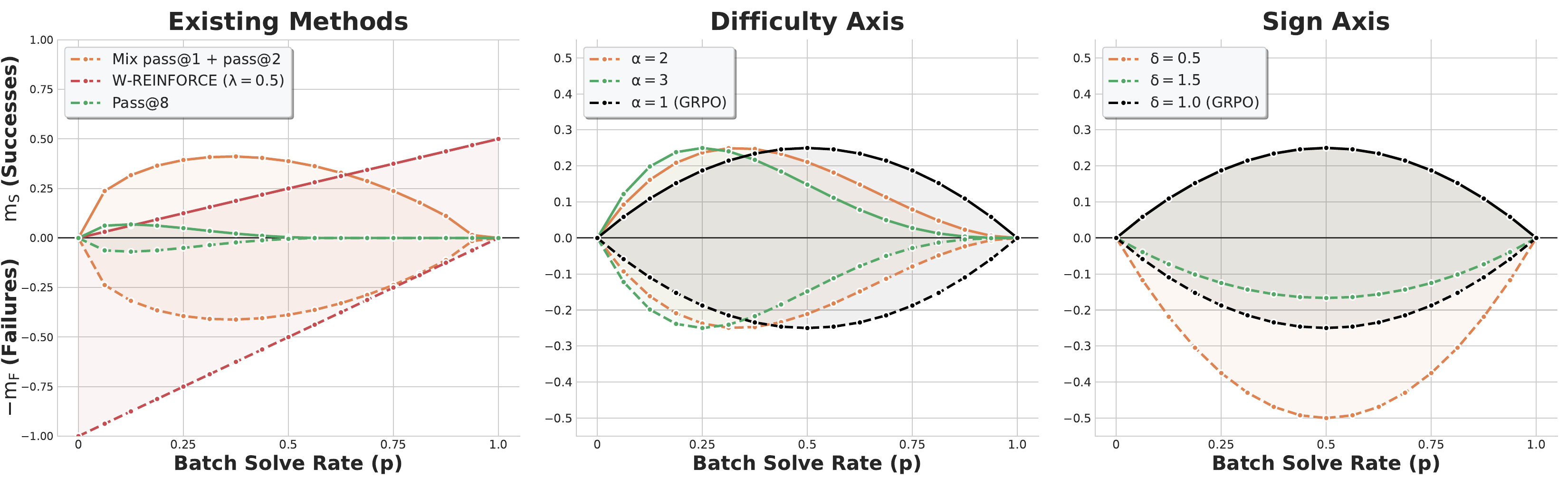}
    \caption{\textbf{Effect of Policy Weights on Batch Level Gradients} based on the estimated solve rate $p$. \textbf{(Left)} Existing methods: mixing pass@$k$ \citep{chen2025pass} and W-REINFORCE \citep{w_reinforce} modify the sign, difficulty focus and scale of advantages. We design weights that isolate \textbf{(Middle)} the difficulty axis with the Power $\alpha$ series $p(1-p)^\alpha$ and \textbf{(Right)} sign axis with the AsymGRPO series, which weights failed trajectories by $\frac{p(1-p)}{\delta}$.}
    \label{fig:weigths_advantages}
\end{figure*}

\begin{table*}[t]
\centering
\caption{\textbf{Policy weights as positive and negative gradient mass} (population limit, binary rewards, solve rate $p$, $q:=1-p$). For each method: positive mass $m_S$, negative mass $m_F$, ascended objective $\phi(p)$ ($\phi'=c$) and baseline $b(p)$. Other symbols are method hyperparameters, defined with the derivations in Appendix~\ref{sec:weight_derivations}. $^\dagger$At finite $G$, MaxRL optimises exactly $\sum_{k<G}\mathrm{pass@}k/k$ (Appendix~\ref{app:maxrl_finite}).}
\label{tab:weights}
\small
\setlength{\tabcolsep}{4pt}
\renewcommand{\arraystretch}{1.2}
\definecolor{ourscolor}{RGB}{255,247,200}
\resizebox{\textwidth}{!}{%
\begin{tabular}{@{}l c c l l@{}}
\toprule
\textbf{Method} & $m_S$ & $m_F$ & \textbf{Objective} $\phi(p)$ & \textbf{Baseline} $b(p)$\\
\midrule
\rowcolor[gray]{0.92}\multicolumn{5}{@{}l}{\emph{Same objective as pass@1, different baseline}}\\
Dr.\ GRPO \citep{drgrpo} & \multicolumn{2}{c}{$pq$} & $p$ & $p$\\
RLOO \citep{ahmadian2024back} & \multicolumn{2}{c}{$pq$} & $p$ & $p$\\
REINFORCE \citep{Williams1992} & $p$ & $0$ & $p$ & $0$\\
W-REINFORCE \citep{w_reinforce} & $\lambda\,p$ & $q$ & $(1+\lambda)p$ & $\tfrac{1}{1+\lambda}$\\
Constant baseline & $(1{-}C)\,p$ & $C\,q$ & $p$ & $C$\\
AsymRL \citep{arnal2026asymmetric} & $p\,(q{-}\delta_A)$ & $q\,(p{+}\delta_A)$ & $p$ & $p+\delta_A$\\
Quantile \citep{dabney2018distributional} & $p\,\mathbf 1[p\le 1{-}\kappa]$ & $q\,\mathbf 1[p>1{-}\kappa]$ & $p$ & $\mathbf 1[p>1{-}\kappa]$\\
MC-GRPO \citep{kim2026mc} & $p\,\mathbf 1[p\le\tfrac12]$ & $q\,\mathbf 1[p>\tfrac12]$ & $p$ & $\mathbf 1[p>\tfrac12]$\\
CoRPO \citep{CorPO} & $p\,(1{-}\max(r_{\min},p))$ & $q\max(r_{\min},p)$ & $p$ & $\max(r_{\min},p)$\\
\rowcolor[gray]{0.92}\multicolumn{5}{@{}l}{\emph{Different objective (prompt reweighting)}}\\
GRPO \citep{deepseek_grpo} & \multicolumn{2}{c}{$\sqrt{pq}$} & $2\arcsin\sqrt p$ & $p$\\
Skew-R \citep{thrampoulidis2025advantage} & \multicolumn{2}{c}{$(pq)^{3/2}$} & $\int_0^p\!\sqrt{u(1{-}u)}\,du$ & $p$\\
Power norm \citep{andrychowicz2020matters} & \multicolumn{2}{c}{$(pq)^{1-\gamma}$} & $\int_0^p (u(1{-}u))^{-\gamma}du$ & $p$\\
F-GRPO \citep{plyusov2026fgrpodontletpolicy} & \multicolumn{2}{c}{$q^{\gamma}\sqrt{pq}$} & -- & $p$\\
MaxRL \citep{tajwar2026maximumlikelihoodreinforcementlearning} & \multicolumn{2}{c}{$q$} & $\log p$\textsuperscript{$\dagger$} & $p$\\
Binary contrastive \citep{greensmith2004variance} & \multicolumn{2}{c}{$p$} & $-\log q$ & $p$\\
Softmax, LME \citep{lme} & \multicolumn{2}{c}{$\propto pq\,(1{+}p(e^\beta{-}1))^{-1}$} & $\propto\log(1{+}p(e^\beta{-}1))$ & $p$\\
ReLU \citep{srinivasan2018actor} & $pq$ & $0$ & $\tfrac12(1-q^2)$ & $0$\\
Symmetric clip \citep{ppo} & $p\min(\epsilon,q)$ & $q\min(\epsilon,p)$ & $\int_0^p\!\big(\min(\epsilon,u){+}\min(\epsilon,1{-}u)\big)du$ & $\min(\epsilon,p)/c$\\
HA-DW \citep{ha_dw} & $\lambda e^{-(p-C_t)}pq$ & $\lambda e^{p-C_t}pq$ & -- & $\frac{p}{p+q\,e^{-2(p-C_t)}}$\\
Pass@$k$ \citep{tang2025optimizing} & $k\,p\,q^{k-1}$ & $0$ & $1-q^k$ & $0$\\
Analytic pass@$k$ \citep{chen2025pass} & \multicolumn{2}{c}{$p\,q^{k}/\sigma_k$} & $\tfrac{2}{k}\arcsin\sqrt{1-q^k}$ & $p$\\
Mix pass@1/@$k$ \citep{chen2025pass} & \multicolumn{2}{c}{$p^2q^{k}/\sigma_k+pq^2/\sigma_1$} & -- & $p$\\
\rowcolor{ourscolor} Power $\alpha$ & \multicolumn{2}{c}{$p\,q^{\alpha}$} & $(1-q^\alpha)/\alpha$ & $p$\\
\rowcolor{ourscolor} Positive power $\alpha$ & \multicolumn{2}{c}{$p^{\alpha}q$} & $p^\alpha/\alpha$ & $p$\\
\rowcolor{ourscolor} Asym.\ power $\alpha$ & $p\,q^{\alpha_s}$ & $p^{\alpha_f}q$ & $\tfrac{1-q^{\alpha_s+1}}{\alpha_s+1}+\tfrac{p^{\alpha_f+1}}{\alpha_f+1}$ & $p^{\alpha_f}/c$\\
\rowcolor{ourscolor} AsymGRPO & $pq$ & $pq/\delta$ & $p+\tfrac12(\tfrac1\delta-1)p^2$ & $\tfrac{p}{p+\delta q}$\\
\rowcolor{ourscolor} FADE & $p\,q^{\alpha}$ & $p\,q^{\alpha}/\delta$ & $\big(1{-}\tfrac1\delta\big)\tfrac{1-q^{\alpha+1}}{\alpha+1}+\tfrac{1}{\delta}\tfrac{1-q^\alpha}{\alpha}$ & $\tfrac{p}{p+\delta q}$\\
\bottomrule
\end{tabular}}
\end{table*}

\section{Where to learn from? Balancing gradient signs and problem difficulty}
\label{sec:problem}

Using our framework, we can see two different axes for weighting the gradient of a prompt relative to others: its difficulty estimate and its overall contribution to positive or negative gradients. To understand the trade-offs of different advantages, we vary one knob at a time: the sign balance at fixed difficulty (Section~\ref{sec:entropy_collapse}), then the difficulty focus at fixed sign balance (Section~\ref{sec:shape}). We will show the baseline $b$ controls entropy and training stability whereas $c$ controls which pass@$k$ improves.

\subsection{Experimental Setup}
\label{sec:experiments}

We compare advantage functions along different weight axes within the PPO clipping framework~\citep{ppo} on competitive programming (LiveCodeBench~\citep{jain2024livecodebench}), math (AIME~2024/2025~\citep{openai2024reasoning}) and software engineering (SWE-bench Verified~\citep{jimenez2024swebenchlanguagemodelsresolve}) tasks, measuring accuracy (pass@$1$~\citep{chen2021evaluating}), diversity (pass@$100$) and learning speed. We summarize our setup below (see full training details in Appendix~\ref{sec:setup_details}):
\begin{itemize}
    \item \textbf{Models:} Qwen~2.5~7B~\citep{qwen2025qwen25technicalreport} and CWM~32B SFT~\citep{cwm} for competitive programming (also evaluated on math), and Qwen3.5-35B-A3B-Base~\citep{team2026qwen3} for software engineering.
    \item \textbf{SFT data (7B only):} DeepSeek-R1~\citep{Guo_2025} traces from OpenCodeReasoning-2~\citep{opencodereasoning2} and OpenMathReasoning~\citep{openmathreasoning}.
    \item \textbf{RL train set:} 25k problems from CodeContests~\citep{li2022alphacode} and TACO~\citep{li2023tacotopicsalgorithmiccode}; R2E-Gym-Subset-Verified~\citep{jain2025r2e} for SWE.
\end{itemize}

We give the full pass@$k$ results on LiveCodeBench v5/v6 and AIME~2024/2025 in Appendix~\ref{sec:more_results}.

\subsection{Sign balance controls overfitting}
\label{sec:entropy_collapse}

\begin{recommendation}
Up-weighting successes collapses entropy and up-weighting failures inflates it: start sign-balanced and add a failure bias only if entropy collapses.
\end{recommendation}

Each RL update combines two forces (Equation~\ref{eq:pop_update}): a positive mass $m_S$ that pushes up successful trajectories and a negative mass $m_F$ that pushes down failed ones. With advantages of the form $A(p):=c(p)(r-b)$, the baseline $b$ sets which way each trajectory is pushed: only $r>b$ adds to $m_S$, so a baseline near $\mathbb E[r]$ barely reinforces correct solutions to prompts the policy already solves. Many recent methods tilt the balance between $m_S$ and $m_F$ (Table~\ref{tab:weights}), but mostly as a side effect of moving the baseline $b$. A moved baseline tilts unevenly across prompts. For example, with AsymRL's offset~\citep{arnal2026asymmetric}, $m_S/m_F=\frac{p(q-\delta_A)}{q(p+\delta_A)}$ stays close to balanced for mid-range $p$ but pushes successes down once $p>1-\delta_A$. Sign and difficulty are therefore entangled, and it is unclear which one drives the reported effects. To isolate the sign axis, we introduce AsymGRPO, which divides negative advantages by $\delta>0$ so that $m_S/m_F=\delta$ for every prompt:
\begin{equation}
A_i^{\mathrm{AsymGRPO}} :=
\begin{cases}
A_i, & A_i \geq 0,\\[2pt]
A_i/\delta, & A_i < 0.
\end{cases}
\label{eq:asym_grpo}
\end{equation}
We take the Dr.~GRPO advantage $A_i = r_i - \bar r$, with $\bar r$ the mean reward of the group, so $\delta = 1$ recovers Dr.~GRPO. With $\delta < 1$ we scale up failures and with $\delta > 1$ we scale them down; for $\delta \to \infty$ failures vanish and we recover online rejection fine-tuning, i.e.\ SFT on the policy's own correct trajectories.

Sweeping $\delta$, we find that entropy tracks the per-sample advantage ratio $|A_+|/|A_-|=\frac{m_S\,q}{m_F\,p}$ (Figure~\ref{fig:entropy}) so the sign of the net advantage sets the direction of the entropy change. Although entropy is not a proxy for learning speed or diversity, we can use it as a proxy for the stability of our RL run. Post SFT the policy already has a bias towards certain correct solutions: hidden representations of correct trajectories are much more concentrated than those of failures (Figure~\ref{fig:entropy_ratios}, left). Reinforcing successes alone keeps pushing along the same few directions, so entropy collapses as the policy memorizes, much like rejection fine-tuning~\citep{xiong2504minimalist,w_reinforce,cui2025entropy}.

Failures are more spread out and act as a source of noise that stops the update from narrowing further. When biasing towards failures $m_F > m_S$ with $\delta<1$, the policy forgets SFT behavior. Pushing a token up concentrates mass on it whereas pushing a token down frees mass with no target. Through the softmax, lowering a sampled token $y$ raises the logit $z_j$ of every other token in proportion to its probability, since $\partial\log\pi(y)/\partial z_j=-\pi(j)$. Tokens that are never sampled receive only these side pushes: down from every success, up from every failure. They cancel when advantages are sign-balanced and otherwise leave a net push $-\big(\sum_i A_i\big)\,\pi(j)$, which lifts the tail when $\delta<1$. Under SGD this lift is negligible, since $\pi(j)$ is tiny in the tail so we are simply adding noise to our policy. However, since Adam rescales gradient steps, every unsampled token $j$ moves by the same amount, set only by the sign of the net advantage:
\begin{equation}
\Delta w_j^{\mathrm{Adam}}\approx-\eta\,\operatorname{sign}\Big(\sum_i A_i\Big)\operatorname{sign}(\bar h),
\label{eq:adam_tail}
\end{equation}
with $w_j$ its output row and $\bar h$ the mean hidden state (Appendix~\ref{sec:proof_rank1}). We can view this tail lift as doing SFT on the tokens the policy never samples (Figure~\ref{fig:entropy_ratios}). This flattens the distribution and raises entropy, which can help when the policy has collapsed. Pushed too far, however, it can block learning, since we undo useful SFT features and try to extract signal from a noisy tail. This is particularly noticeable in smaller models, where the update collapses to a rank-1 direction (Figure~\ref{fig:entropy_ratios}, right) and every tail token moves along the same $\operatorname{sign}(\bar h)$ regardless of the prompt, so the step no longer carries information about what separates successes from failures (Figure~\ref{fig:pass_at_k_scheduler}).

\paragraph{Choosing the sign balance.} The right balance depends on how much we can trust the policy's own successes, which in turn depends on model size and on the hidden features learned during SFT. When correct solutions are low quality, as in smaller models, they concentrate on a few brittle patterns and reinforcing them speeds up collapse. A failure bias ($\delta<1$) then keeps the policy exploring. Once correct and failed representations are comparably spread out, neither side gives a clearly better signal and sign balance ($\delta=1$) is the safest choice. Consistent with this, our 7B runs use a stronger failure bias ($\delta\in\{0.3,0.5\}$) than our 32B run ($\delta=0.8$, Table~\ref{tab:rank1_main}).

\begin{figure}[t]
    \centering
    \includegraphics[width=\linewidth]{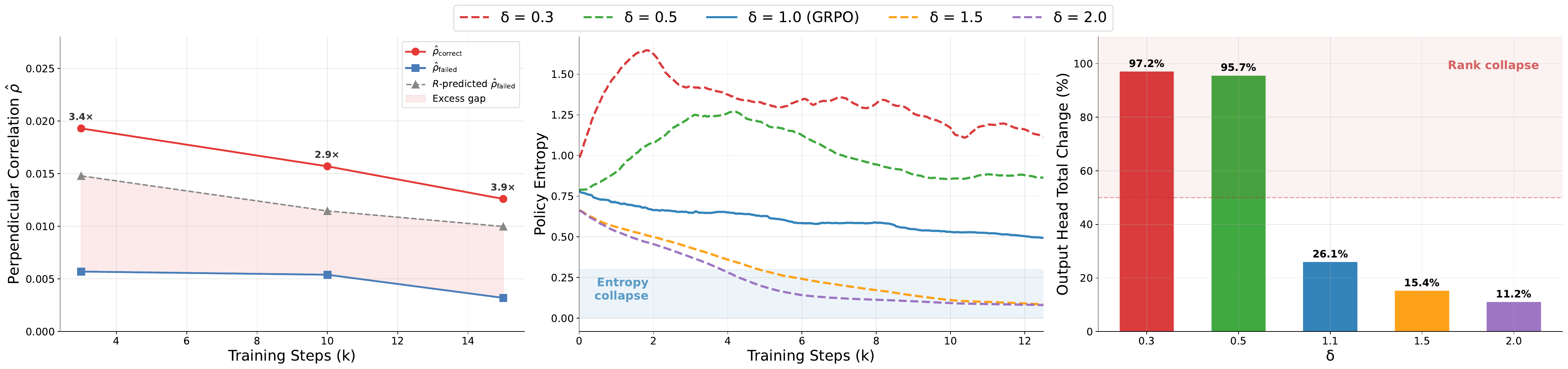}
    \caption{Scaling negative gradients of Dr.~GRPO by $\frac{1}{\delta}$ ($\delta < 1$: up, $\delta > 1$: down) on Qwen 2.5 7B. \textbf{(Left)} For AsymGRPO with $\delta = 0.5$, correct samples are more correlated than failures in the hidden-state subspace orthogonal to the leading update direction (Appendix~\ref{sec:weight_space_details}). \textbf{(Middle)} Over-reinforcing ($\delta > 1$) causes entropy collapse; \textbf{(Right)} over-suppressing ($\delta < 1$) causes rank-1 update collapse.}
    \label{fig:entropy_ratios}
\end{figure}

\begin{figure}[t]
    \centering
    \includegraphics[width=\linewidth]{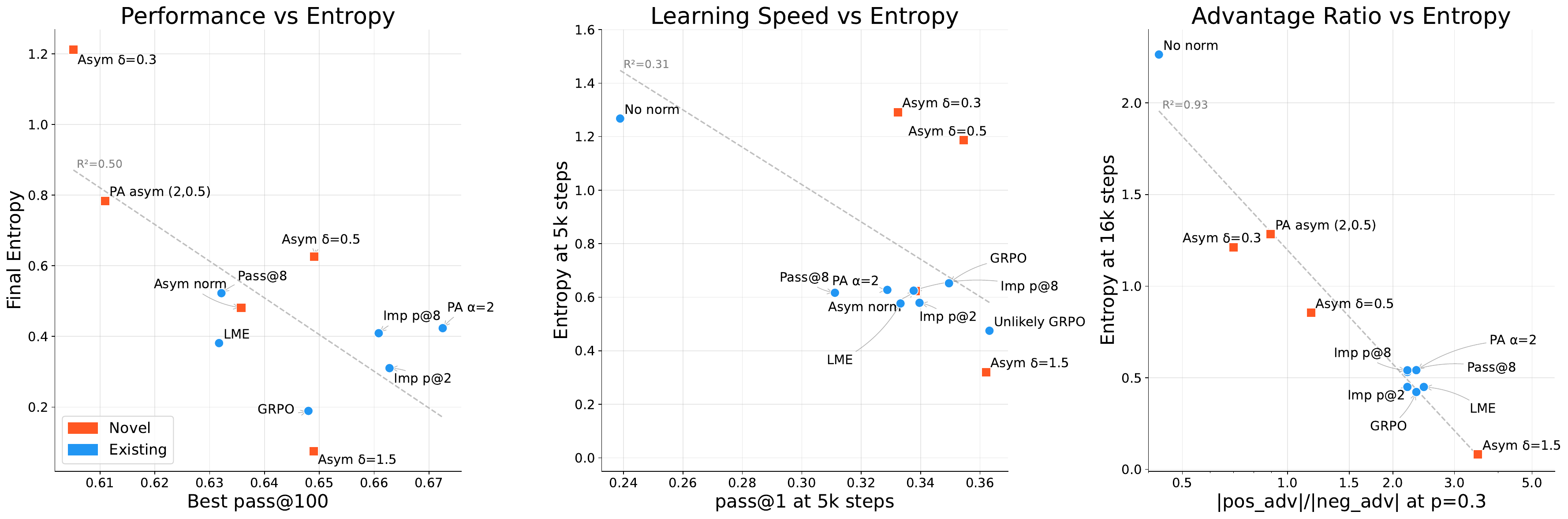}
    \caption{Higher entropy does not mean more output diversity (left, the relation is negative), and lower entropy does not mean faster learning (middle). Instead, entropy tracks the advantage ratio $|A_+|/|A_-| = m_S(1-p)\,/\,m_F\,p$ (right), with $m_S$, $m_F$ the mass of successful and failed trajectories and $p$ the solve rate. Sign-balanced weights ($m_S=m_F$) sit at $q/p\approx2.33$.}
    \label{fig:entropy}
\end{figure}

\subsection{Harder problems decide how pass@$k$ scales}
\label{sec:shape}

\begin{recommendation}
Easy prompts learn fastest and flatten pass@$k$; reweight so all difficulties learn equally fast.
\end{recommendation}

So far we treated all prompts alike, but a reasoning dataset mixes prompts the model almost always solves with prompts it almost never solves. Write $p_x$ for the solve rate of prompt $x$. Since $\nabla_\theta p_x=p_x(1-p_x)(\nabla_S-\nabla_F)$, the update of Equation~\ref{eq:pop_update} is simply $c(p_x)\,\nabla_\theta p_x$: the prompt weight $c$ rescales the pass@$1$ gradient of each prompt. Integrating $c$ gives the objective the advantage ascends,
\begin{equation}
J(\theta)=\mathbb{E}_{x\sim\mathcal{D}}\big[\phi(p_x)\big],\qquad \phi'=c,
\label{eq:multitask}
\end{equation}
where prompts of similar difficulty $p$ act as one task, since the policy often solves them with similar approaches. Multi-task training works when tasks help each other, e.g.\ when solving easy prompts transfers to hard ones, but \citet{barakat2026pass} show that gradients from prompts of different difficulty interfere in LLM post-training. Even without interference, the tasks do not learn at the same speed. A prompt the model never solves gives no successes to reinforce, so the gradient mass $m_S(p)$ vanishes at $p=0$ and peaks at a difficulty focus $\arg\max_p m_S(p)$ (Figure~\ref{fig:weigths_advantages}). Prompts near the focus learn fastest. For Dr.~GRPO the focus is $p=\tfrac12$, so on a dataset of mostly hard prompts, easier prompts learn faster. Existing advantages counter this by picking a side: pass@$k$ moves the focus to $p=1/(k{+}1)$~\citep{chen2025pass}, and MaxRL sets $c=1/p$~\citep{tajwar2026maximumlikelihoodreinforcementlearning}. Neither says how much emphasis is enough. To move the focus continuously while keeping the sign balanced, we introduce \textbf{Power $\alpha$}, which multiplies the Dr.~GRPO advantage by $(1-\bar p)^{\alpha-1}$, with $\bar p$ the group's empirical solve rate. Its masses are $m_S=m_F=pq^{\alpha}$, so its focus sits at $p=1/(1{+}\alpha)$, e.g.\ $p=\tfrac13$ for $\alpha=2$ (Figure~\ref{fig:weigths_advantages}; Appendix~\ref{sec:weight_derivations}).

Instead of choosing a task, we ask which advantage lifts pass@$k$ at every $k$. Pass@$k$ is the probability that at least one of $k$ independent attempts succeeds, averaged over prompts: $\text{pass@}k=\mathbb E_x\big[1-(1-p_x)^k\big]$. It measures improvement across all tasks at once, because each budget $k$ looks at a different slice of difficulty: easy prompts are solved at low $k$, hard ones only at high $k$, and as $k\to\infty$ pass@$k$ counts every prompt the policy can solve at all ($p_x>0$). When solve rates near zero follow a power law, so does the curve~\citep{brown2024largelanguagemonkeysscaling,schaeffer25a}:
\begin{equation}
1-\text{pass@}k\;\propto\;k^{-b},
\label{eq:passk_law}
\end{equation}
where the slope $b$ (not to be confused with the baseline) measures how few prompts sit near $p=0$ (Appendix~\ref{sec:proof_theorem}). We fit $b$ over the measured range of $k$, so it is an effective slope that matches this tail exponent only asymptotically. \citet{levi2026learningshrinkshardtail} showed SFT lifts the hardest prompts most and therefore steepens the pass@$k$ curve. On the other hand, Dr.~GRPO spends its updates on prompts that are already partly solved. It makes easy prompts easier and leaves hard ones behind (Figure~\ref{fig:solved_bins}), which flattens the curve (Figures~\ref{fig:swe_passk} and~\ref{fig:intuition}). Is this flattening inherent to RL, or a consequence of which prompts the advantage favours?

\begin{figure}[t]
    \centering
    \includegraphics[width=0.8\linewidth,trim=28 8 36 13,clip]{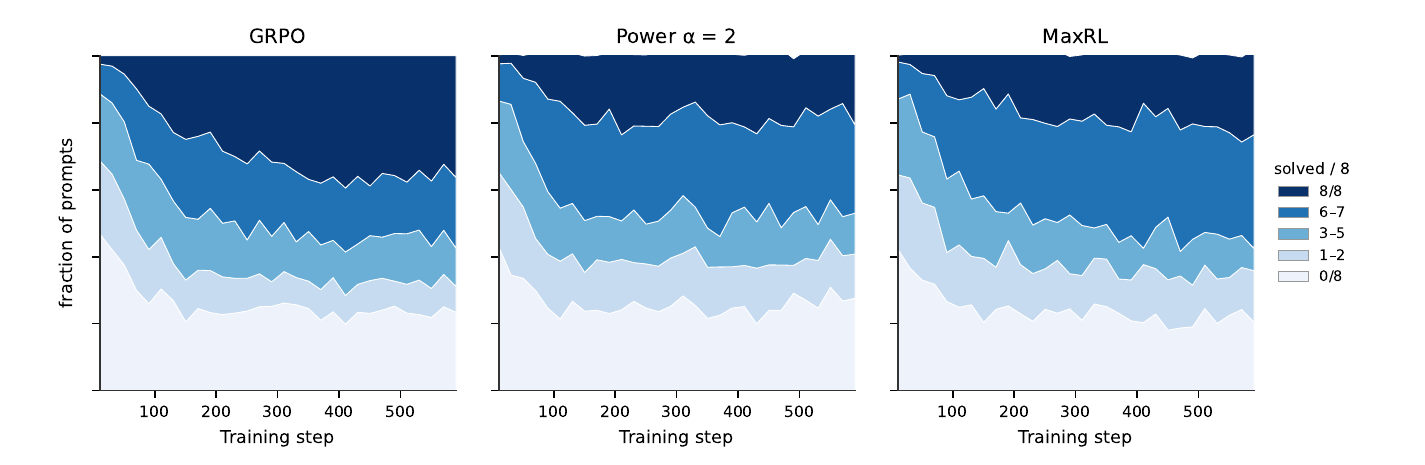}
    \caption{\textbf{Training Prompt Solve Rates per Advantage} on the SWE training set (R2E-Gym) with $8$ samples per prompt. Dr.~GRPO increases sample accuracy by turning prompts with average solve rates into high solve rates. Pass@2 and MaxRL focus on harder problems.}
    \label{fig:solved_bins}
\end{figure}

To answer this, we need to know how an advantage moves each prompt, and how that movement changes the slope $b$. We track every prompt by its log-odds of success, $\ell=\ln\frac{p}{1-p}$, and summarize a window of training with three numbers:
\begin{equation}
\ell_{\text{after}}=\underbrace{\Lambda}_{\text{shift}}+\underbrace{\nu}_{\text{stretch}}\,\ell_{\text{before}}+\underbrace{\sigma z}_{\text{noise}},\qquad z\sim\mathcal{N}(0,1).
\label{eq:ssc}
\end{equation}
Log-odds are the natural scale for this: near $p=0$ we have $\ell\approx\ln p$, so the affine map sends $p\mapsto e^{\Lambda}p^{\nu}$. It is the simplest change that keeps pass@$k$ a power law, and it rescales the effective slope by roughly $1/\nu$ (Appendix~\ref{app:rl_scaling}). Each parameter then does one thing to the pass@$k$ curve:
\begin{itemize}
    \item \textbf{Shift $\Lambda$:} the shared improvement that transfers from all updates. It lifts every prompt equally, raising pass@$k$ at every $k$ without changing the slope; in the tail it acts as a multiplicative budget increase $k\mapsto k\,e^{\Lambda}$.
    \item \textbf{Stretch $\nu$:} says whether easy ($\nu>1$) or hard ($\nu<1$) prompts moved more, and alone sets the slope.
    \item \textbf{Noise $\sigma$:} prompts do not all gain exactly $\Lambda$: some gain more, some less, some nothing. This spread around the average gain is the noise. In our runs it grows with the shift ($\sigma\propto\Lambda$) and fades as training slows.
\end{itemize}
Prompts move mostly through this transfer rather than through their own gradient: SWE-bench Verified tasks never seen in training improve as much as trained prompts (pass@$1$ $0.44\to0.77$ vs.\ $0.44\to0.74$), and prompts on their first training visit already match revisited ones.

To preserve the slope, learning speed must be balanced across difficulties: with $\nu=1$, every prompt gains the same log-odds $\Lambda$, whatever its starting solve rate. The advantage controls this balance through its focus. When most prompts sit below the focus, easier prompts are closer to it and learn faster, so the distribution stretches ($\nu>1$) and the curve flattens. Given the advantage's weight and the current solve rates, this predicts $\nu$ with no free parameters, and the predictions match the fits (Figure~\ref{fig:affine_law}; Appendix~\ref{app:rl_scaling}). Appendix~\ref{app:passk_recipe} sketches how to use this to forecast the pass@$k$ of a run from its advantage and the solve rates of the base model. Dr.~GRPO focuses on $p=\tfrac12$, above most of our prompts, so it stretches (fitted $\nu=1.23$, predicted $1.19$) and its slope on SWE-bench Verified falls from the base model's $0.39$ to $0.34$ (Figure~\ref{fig:swe_passk}).

In other words, we need to estimate how fast each prompt learns and scale it by the inverse. For example MaxRL does this with a reweighting $c = 1/p$ which upweights hard prompts. Upweighting hard prompts has a cost, however, because their signal is noisy and rests on very few samples. The useful signal also shrinks as the policy improves, since fewer groups contain both successes and failures and the prompts that remain unsolved are often noisy or out of reach. A fixed focus on hard prompts therefore spends a growing share of each update on noise. In practice, focusing on moderately hard prompts balances the two: on LiveCodeBench, weights focused on $p\in[0.25,0.4]$ reach the best pass@$100$ (Figure~\ref{fig:difficulty_p_max}), and on SWE-bench Verified both MaxRL ($b=0.37$) and Power $\alpha{=}2$ ($b=0.39$) roughly keep the base slope of $0.39$ (Figure~\ref{fig:swe_passk}).

\begin{figure*}[t]
    \centering
    \includegraphics[width=\linewidth]{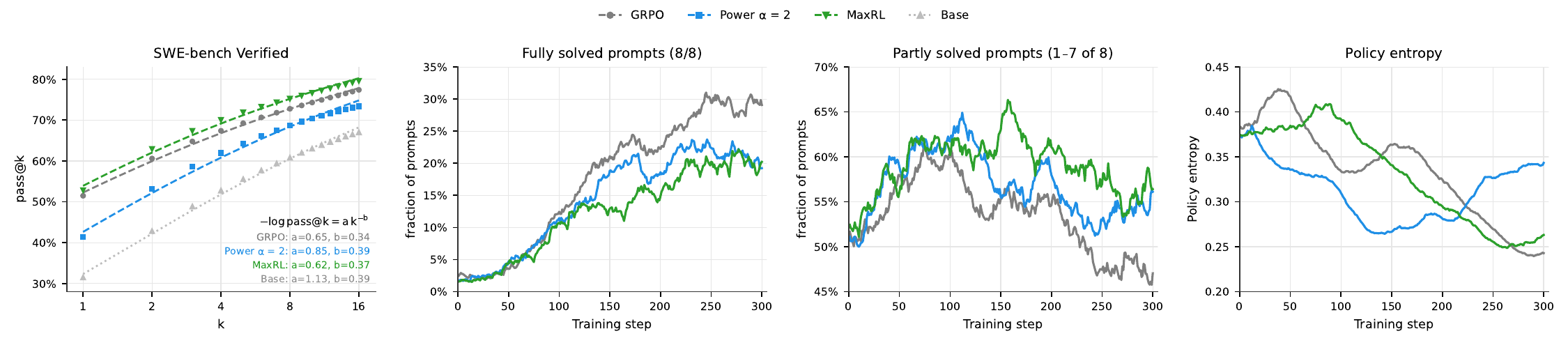}
    \caption{\textbf{RL on SWE-bench Verified.} Dr.~GRPO flattens the pass@$k$ curve (slope $b$: $0.39\to0.34$), while Power $\alpha{=}2$ ($b=0.39$) and MaxRL ($b=0.37$) roughly keep the base slope as they solve more hard problems.}
    \label{fig:swe_passk}
\end{figure*}

\begin{figure}[t]
    \centering
    \begin{minipage}[c]{0.55\linewidth}
        \centering
        \includegraphics[width=\linewidth]{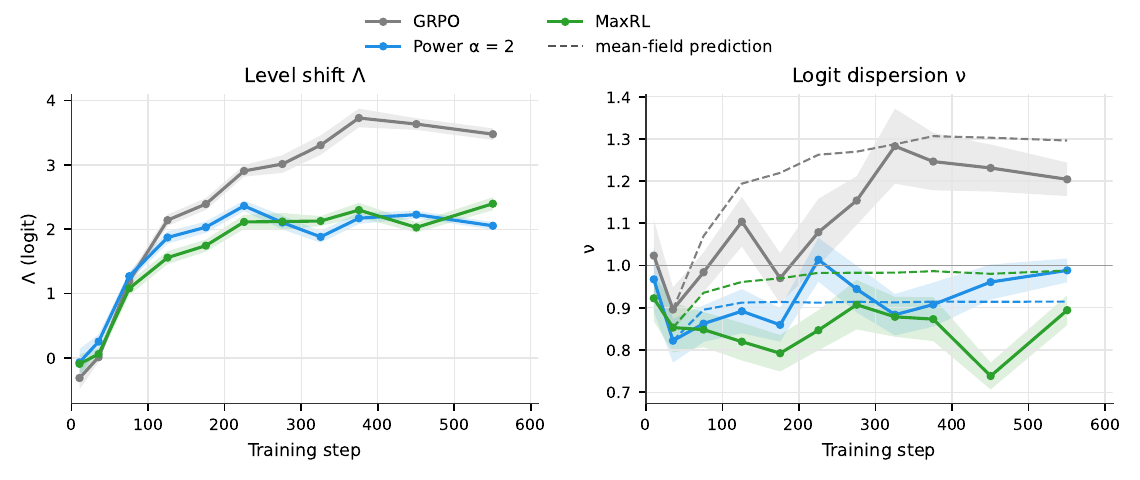}
    \end{minipage}\hfill
    \begin{minipage}[c]{0.42\linewidth}
    \caption{\textbf{Agentic RL on SWE.} The shift $\Lambda$ lifts all prompts equally, raising pass@$k$ at every $k$; a stretch $\nu>1$ (right, shown as logit dispersion) means easy prompts moved more than hard ones. Dashed: mean-field predictions of our scaling law (Appendix~\ref{app:rl_scaling}). \textbf{Dr.~GRPO shifts the most but stretches ($\nu>1$), while Power $\alpha{=}2$ and MaxRL stay at $\nu\lesssim1$.}}
    \label{fig:affine_law}
    \end{minipage}
\end{figure}

\begin{figure}[t]
    \centering
    \begin{minipage}[c]{0.57\linewidth}
        \centering
        \includegraphics[width=\linewidth]{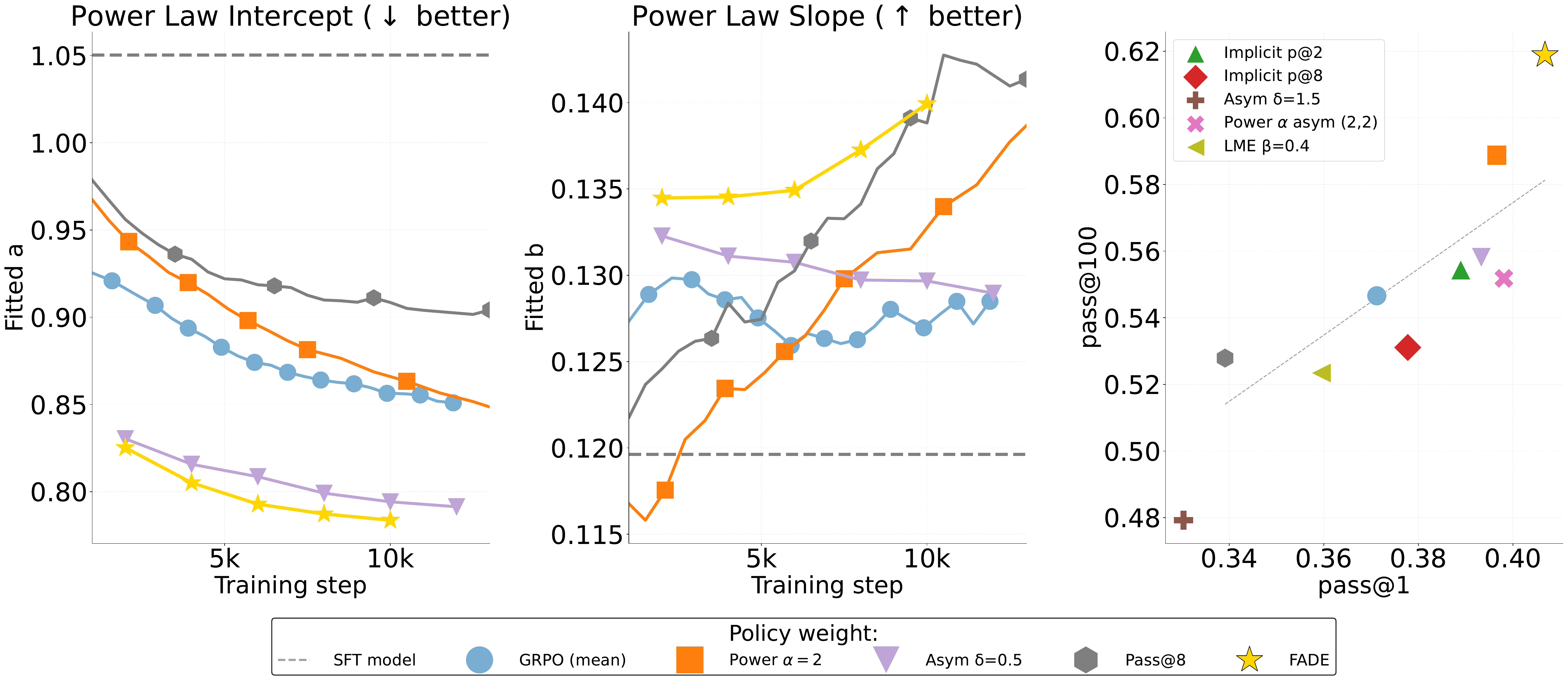}
    \end{minipage}\hfill
    \begin{minipage}[c]{0.41\linewidth}
        \centering
        \includegraphics[width=\linewidth]{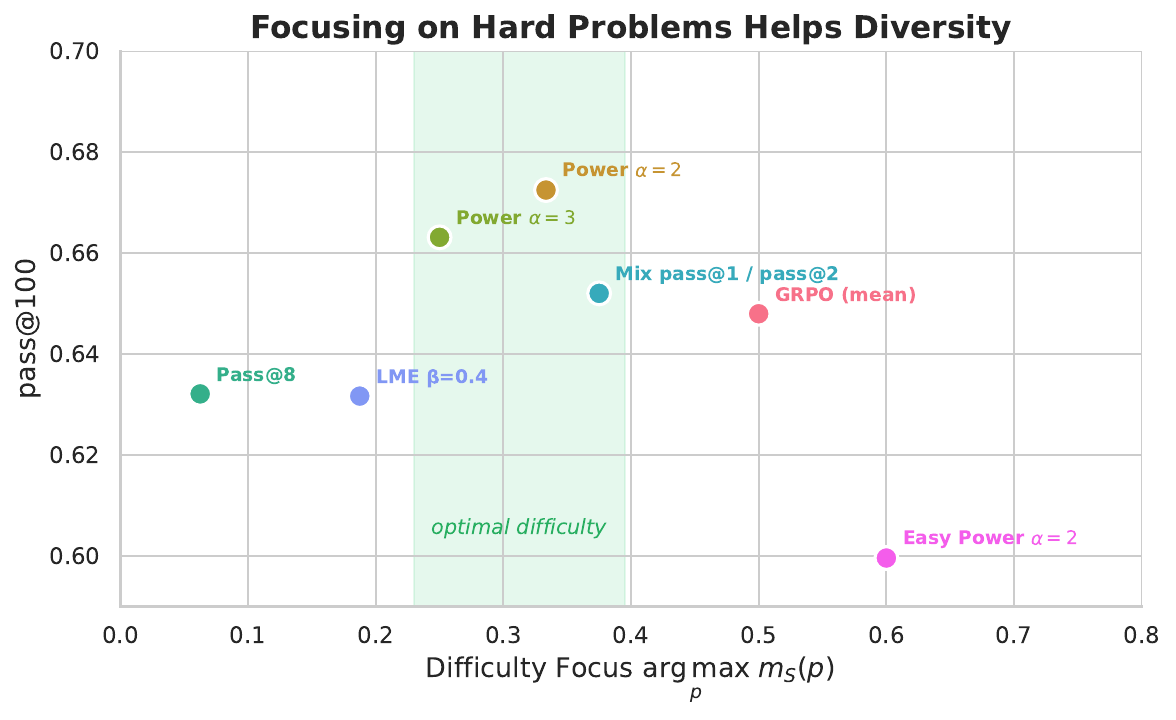}
    \end{minipage}
    \caption{\textbf{Accuracy vs.\ diversity per policy weight on LiveCodeBench v6.} \textbf{(Left)} Fitted level $a$ and effective slope $b$ of $\text{pass@}k=\exp(-a(k+k_0)^{-b})$ over training, and pass@$100$ vs.\ pass@$1$: diverse, accurate advantages lower $a$ while keeping or raising $b$, maximizing both. \textbf{(Right)} Pass@$100$ versus the difficulty focus $\arg\max_p m_S(p)$ of each policy weight.}
    \label{fig:fitted_coefs}
    \label{fig:difficulty_p_max}
\end{figure}

\section{The FADE Scheduler}
\label{sec:scheduled-advantages}

The previous sections show that the best weight on each axis depends on the current state of the policy, so no fixed weight is right throughout training. A failure bias helps only while the policy's successes are low quality and entropy is collapsing; otherwise sign balance is best. On the difficulty axis, to improve pass@$k$ across all $k$ we need uniform learning speed across all tasks. The best focus therefore moves as solve rates rise.

To maximize gradient efficiency, we propose to adapt the policy weight to online learning dynamics. Combining AsymGRPO and Power $\alpha$, we propose \textbf{Focal Advantage with Dynamic Entropy} (FADE) which adapts the policy weight to the moving average of solve rate $\hat{p}$ and entropy $\hat{H}$:
\begin{equation}
  A_i = (1-\bar{r})^{\alpha-1} \cdot
  \begin{cases}
    r_i - \bar{r} & \text{if } r_i \geq \bar{r}, \\[4pt]
    \dfrac{r_i - \bar{r}}{\delta} & \text{if } r_i < \bar{r},
  \end{cases}
  \quad
  \alpha = \frac{1 - \hat{p}}{\hat{p}},
  \quad
  \delta = \operatorname{clip}\big(1 + \hat{H} - H^*,\,0.3,\,1\big),
  \label{eq:fade_advantage}
\end{equation}
where $\bar r$ is the mean reward of the prompt's $G$ rollouts, $\hat p$ and $\hat H$ are exponential moving averages of the batch solve rate and policy entropy, and $H^*$ is a target entropy. In practice, setting $H^*$ to about half the initial entropy works well.

FADE accelerates training by taking the largest gradient steps the model can use, adapting to changes in its learning capacity and average solve rate. With $\alpha=(1-\hat p)/\hat p$, the weight $p(1-p)^\alpha$ peaks at the current average solve rate $\hat p$. Early in training, hard problems are the frontier of what the model can learn, so FADE focuses on learning to reason about them. Later, problems that are still unsolved will likely never be learned, either because the policy is saturated or because they are noisy, so $\alpha$ decays and FADE switches to exploiting what the model already knows. Meanwhile, $\delta$ pushes back against overfitting: when entropy falls below the target $H^*$, $\delta<1$ upweights failures; otherwise $\delta=1$ and the update stays sign-balanced. FADE outperforms the best static advantages at every pass@$k$ (Figure~\ref{fig:pass_at_k_scheduler}); controller dynamics, hyperparameters and ablations are in Appendix~\ref{sec:appendix_fade}.

\section{Related Work}

\textbf{Policy-gradient weights and diversity collapse.}
Sampling $k$ candidates and keeping the best is now a standard way to improve LLM reasoning \citep{wangself, lightman2024lets}, and the resulting pass@$k$ curves follow a power law set by the distribution of per-problem solve rates \citep{brown2024largelanguagemonkeysscaling, schaeffer25a}. RL fine-tuning, however, tends to sharpen the policy: pass@$1$ improves while pass@$k$ degrades as the model collapses onto a few solutions \citep{huang2024self, karan2025reasoning, levi2026learningshrinkshardtail}, a trade-off \citet{barakat2026pass} trace to opposing gradient forces. Most proposed fixes change the policy-gradient weight: beyond those in Table~\ref{tab:weights}, these include pass@$k$ reweighting \citep{amini2024variational, chow2024inference, walder2025pass}, sign reweighting \citep{he2025rewarding}, top-$k$ redistribution \citep{peng2025simko}, and entropy-based rescaling of the GRPO weight \citep{zhao2026aem}. Others act outside the weight, through the clip bound \citep{mai_thinking_1}, close in spirit to FADE's $\delta$ scheduler, or through in-context guidance \citep{qu2026pope}. \citet{cheng2026cancellation} further show that in GRPO, failed trajectories provide credit assignment for positive tokens. Recent theory asks what these weights actually optimize. \citet{davis2025objectivereasoningreinforcementlearning} show that different objectives amount to reweighting the log-likelihood loss, with rejection sampling \citep{xiong2504minimalist}, REINFORCE, and GRPO as special cases. Similarly, \citet{thrampoulidis2025advantage} show that changing the advantage amounts to training on a modified reward, which puts pass@$k$ methods in a single framework. We generalize these analyses to other advantages and test them to understand which choices matter with LLMs at scale.

\textbf{Curriculum Learning for RL Post-Training.}
Recent work on data-efficient RLVR converges on a consistent finding: prompts of intermediate difficulty, where $p_{\text{success}}(x) \approx 0.5$, maximize gradient informativeness \citep{gao2025pcl, zeng2025cures}. Several methods have proposed online prompt selection to exploit this. To predict a prompt's success rate, MoPPS \citep{qu2025mopps}, GPS \citep{qu2025gps} and CurES \citep{zeng2025cures} use a latent Bernoulli variable learned on gradient heuristics. PCL \citep{gao2025pcl} uses a concurrent value model. Our proposed FADE can be seen as an in-loss curriculum that mirrors their data curriculum, but is driven purely by training signals, similar to prioritized sampling based on TD-error or learning progress \citep{schaul2015prioritized, graves2017automated}.

\textbf{Weight-space analysis of RL training.}
\citet{kumarimplicit} studied rank collapse in Q-learning with gradient descent, attributing it to bootstrapping. In contrast, we are interested in the rank of the update weight since in RL with LLMs we have relatively few changes in the model weights. \citet{cai2025predictability} observed rank-1 dominance of the parameter update matrix, more pronounced in RL than in SFT or distillation. \citet{ye2026implicit} showed that non-rank-1 directions encode out-of-domain abilities beyond reasoning, framing rank-1 collapse as a form of overfitting. \citet{moalla2024no} showed update weights become lower rank in high-entropy regimes correlating with our observations for $\delta < 1$.

\textbf{Multi-Task RL} trains a shared representation to maximize rewards over $N$ tasks that can benefit from each other's gradients. In MT-RL, tasks with different reward magnitudes dominate the joint gradient, a problem addressed by per-task adaptive normalization such as PopArt \citep{hessel2019multi, van2016learning} and IMPALA \citep{espeholt2018impala}. We can view GRPO as an automatic downweighting via variance which resembles the uncertainty weighting of \citet{kendall2018multi}. MT-RL also suffers from gradient conflicts where improving one task degrades another, addressed by gradient projection methods like PCGrad \citep{yu2020gradient} and CAGrad \citep{liu2021conflict}. The same conflict appears in LLM RL where pass@$k$ can encourage rare reasoning approaches not suitable for easy tasks.

\begin{figure}[!h]
    \centering
    \includegraphics[width=0.60\linewidth]{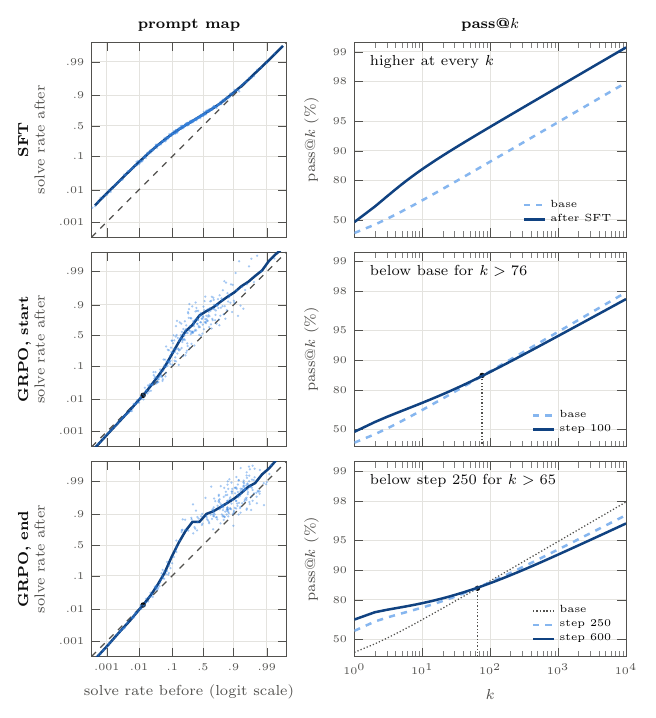}
    \caption{\textbf{RL leaves the hardest prompts behind.} Simulated prompt maps: each prompt's solve rate before vs.\ after a stretch of training, and the resulting pass@$k$. SFT lifts the hardest prompts most, so pass@$k$ rises at every $k$. Dr.~GRPO moves the middle most and falls below the base model at large $k$ (Appendix~\ref{app:intuition_sim}).}
    \label{fig:intuition}
\end{figure}

\section{Limitations}
\label{sec:limitations}
Our analysis assumes binary, terminal rewards ($r \in \{0,1\}$) the standard setting for code and math verification. The $m_S / m_F$ decomposition extends naturally to non-binary rewards but not to finer-grained credit assignment such as process reward models. Scoring intermediate reasoning steps may reduce the role of trajectory-level policy weights relative to per-step signals. Second, our framework captures advantages that depend only on the reward vector $\mathbf{r}$ to compute weights. Methods that additionally condition on the policy's own likelihood (e.g., KL-regularized objectives, DPO-style \citep{rafailov2023direct} implicit rewards) fall outside the current decomposition and would require extending the $S/F$ partition to account for likelihood-dependent weighting. Third, the pass@$k$ analysis of Section~\ref{sec:shape} is mean-field: it assumes similar gradient norms across prompts.

\section{Conclusion}
We provide an empirical analysis of advantage functions in RL to understand what drives training instability, diversity and learning speed. Focusing on advantages of the form $A:= c(p)(r-b)$ we show they can be decomposed along two axes: which prompts their gradients focus on and the net sign of each gradient. The sign axis controls training stability: added negative mass makes us learn from non-sampled tokens and added positive mass from sampled ones. On the difficulty axis, prompts near the advantage's focus learn fastest, so with mostly hard prompts Dr.~GRPO makes easy prompts easier, leaves hard ones behind and flattens the pass@$k$ curve. Keeping learning speed equal across difficulties preserves the slope, and since solve rates rise during training, the right focus moves with them. FADE follows both trends online, moving the focus with the average solve rate and adding a failure bias only when entropy drops. In math and competitive programming tasks, FADE learns faster than static advantages and achieves higher pass@$k$ at every $k$.

Looking ahead, our analysis points to value functions as a natural next step. The group mean we use as a baseline is a Monte Carlo estimate of the value $V(x)=p$, so both axes are really functions of a value estimate. This estimate is noisiest on the hard prompts we want to upweight, and a learned value function could make it smoother. Its errors, however, would act as a hidden sign bias: overestimating $p$ makes the net advantage negative and inflates entropy, while underestimating it drives collapse. Carrying the two axes over to actor-critic methods would thus mean monitoring the critic's bias like $\delta$, and could let a per-state value $V(s_t)$ set the difficulty focus at intermediate steps of a long trajectory rather than only per prompt.

\newpage
\bibliographystyle{assets/plainnat}
\bibliography{paper}

\clearpage
\newpage
\beginappendix
\section{Appendix Outline}
\label{sec:appendix_outline}

\begin{center}
\small
\setlength{\tabcolsep}{4pt}
\renewcommand{\arraystretch}{1.2}
\begin{tabular}{@{}>{\raggedright\arraybackslash}p{0.31\linewidth}>{\raggedright\arraybackslash}p{0.46\linewidth}>{\raggedright\arraybackslash}p{0.17\linewidth}@{}}
\toprule
\textbf{Section} & \textbf{What it contains} & \textbf{Paper section} \\
\midrule
\hyperref[sec:training_details]{\ref*{sec:training_details}. Training and Evaluation Details} & Experimental setup, hyperparameters and full evaluation results. & Section~\ref{sec:problem} \\[4pt]
\hyperref[sec:weight_derivations]{\ref*{sec:weight_derivations}. Deriving Positive and Negative Weights} & Gradient-mass derivations for the advantages in Table~\ref{tab:weights}. & Section~\ref{sec:framework} \\[4pt]
\hyperref[sec:proofs]{\ref*{sec:proofs}. Proofs and Additional Theory} & The optimal baseline and the role of trajectory length and token entropy. & Section~\ref{sec:framework} \\[4pt]
\hyperref[sec:weight_space_details]{\ref*{sec:weight_space_details}. Proofs for Signed Advantages} & Net advantage, tail drift under Adam, predicted entropy and weight-space evidence. & Section~\ref{sec:entropy_collapse} \\[4pt]
\hyperref[sec:proof_theorem]{\ref*{sec:proof_theorem}. Pass@$k$ Scaling Law: Proofs, Fits and Simulations} & Scaling-law proofs, the effect of difficulty weighting, empirical fits and simulations. & Section~\ref{sec:shape} \\[4pt]
\hyperref[sec:appendix_fade]{\ref*{sec:appendix_fade}. FADE Scheduler Additional Results} & Algorithms, controller dynamics, entropy-target analysis and ablations. & Section~\ref{sec:scheduled-advantages} \\
\bottomrule
\end{tabular}
\end{center}

\section{Training and Evaluation Details}
\label{sec:training_details}

\subsection{Experimental Setup}
\label{sec:setup_details}

All evaluations sample at temperature~$1.0$ and top-p~$1.0$, with reasoning budgets from 8k to 30k tokens for competitive programming and math, and up to $128$ turns in a 131k-token context for SWE. Qwen~2.5~7B lacks chain-of-thought ability, so we first SFT it on DeepSeek-R1 traces (Section~\ref{sec:experiments}). For CWM~32B, we filter out easy training problems so that both models start near the same difficulty frontier. The SWE runs train on the 4,522 tasks of R2E-Gym-Subset-Verified~\citep{jain2025r2e}. AIME~2024/2025 is unseen during RL. Table~\ref{tab:training-setup} lists all hyperparameters, and Appendix~\ref{sec:more_results} gives the full results.

\subsection{Infrastructure and Hyperparameters}
\label{sec:training_app}

We ran the code experiments on an internal RL codebase and the SWE experiments in prime-rl~\citep{prime-rl}. For the code runs, similar to \citet{RLEF,rlhf_faster}, the framework is asynchronous and distributed: generation GPUs sample rollouts, a separate CPU cluster runs the unit tests, and trainer GPUs update the policy, whose weights are periodically sent back to the generators. Because reasoning traces vary in length, we batch by tokens rather than by samples. Table~\ref{tab:training-setup} lists all settings.

All advantages are used within the PPO clipping framework,
\begin{equation*}
    L_t(\theta) = -\min\Big(r_t(\theta)\,\hat{A}_i,\ \text{clip}\big(r_t(\theta),\, 1{-}\varepsilon_{\text{low}},\, 1{+}\varepsilon_{\text{high}}\big)\,\hat{A}_i\Big),\qquad r_t(\theta) = \frac{\pi_\theta(a_t \mid s_t)}{\pi_{\theta_{\text{old}}}(a_t \mid s_t)},
\end{equation*}
with the clip ranges of Table~\ref{tab:training-setup}. The loss is a token sum with a constant normaliser, $\mathcal{L}(\theta) = \frac{1}{T_{\max}} \sum_{i} \sum_{t \in \mathcal{A}_i} L_t(\theta)$, where $\mathcal{A}_i$ is the set of model-generated token positions in rollout $i$ and $T_{\max}$ is the maximum token budget per micro-batch.

\begin{table*}[t]
\centering
\footnotesize
\setlength{\tabcolsep}{4pt}
\caption{RL training setup.}
\label{tab:training-setup}
\begin{tabular}{@{}>{\raggedright\arraybackslash}p{2.4cm}>{\raggedright\arraybackslash}p{3.3cm}>{\raggedright\arraybackslash}p{3.3cm}>{\raggedright\arraybackslash}p{4.0cm}@{}}
\toprule
 & \textbf{Qwen2.5 7B (code)} & \textbf{CWM 32B (code)} & \textbf{Qwen3.5-35B-A3B (SWE)} \\
\midrule
\multicolumn{4}{@{}l}{\textit{Model and data}} \\
Initialization & Qwen2.5-7B, SFT & CWM SFT & Qwen3.5-35B-A3B-Base \\
Training data & competitive programming & competitive programming (hard subset) & R2E-Gym, 20 task shards \\
Interaction & single turn, reasoning & single turn, reasoning & multi-turn agent, $\le 128$ turns \\
Max sequence length & 8{,}192 & 32{,}768 & 131{,}072 (4{,}096 per turn) \\
Reward & binary (tests pass) & binary (tests pass) & binary (hidden tests pass) \\
\midrule
\multicolumn{4}{@{}l}{\textit{Compute}} \\
GPUs & 64 & 256 & 32 (GB300) \\
Trainer : generation GPUs & 32 : 32 & 88 : 168 & 8 : 24 \\
Trainer parallelism & FSDP & FSDP, tensor parallel 8 & FSDP, expert parallel 4, context parallel 2 \\
\midrule
\multicolumn{4}{@{}l}{\textit{Optimization}} \\
Optimizer & AdamW & AdamW & AdamW \\
$(\beta_1, \beta_2)$, $\epsilon$ & $(0.9, 0.95)$, $10^{-8}$ & $(0.9, 0.95)$, $10^{-8}$ & $(0.9, 0.999)$, default \\
Weight decay & 0.1 & 0.1 & 0.01 \\
Gradient clipping & 1.0 & 0.1 & 1.0 \\
Precision & bf16 & bf16 & fp32 optimization and reduction \\
Policy loss & clipped ratio $[0.8, 1.2]$, no KL & clipped ratio $[0.8, 1.25]$, no KL & DPPO masking (0.2, 0.2), KL term $\tau = 10^{-3}$ \\
\midrule
\multicolumn{4}{@{}l}{\textit{Rollouts}} \\
Group size $G$ & 16 & 8 & 8 \\
Batch & 8{,}192 tokens per trainer rank & 32{,}768 tokens per rank, 2 accumulation steps & $5\times10^{6}$ tokens per step \\
Off-policy bound & 30 steps & 30 steps, weights synced every 4 steps & 48 steps \\
\bottomrule
\end{tabular}
\end{table*}

\subsection{Advantage Ratios}
\label{sec:appendix_entropy}

In Section~\ref{sec:entropy_collapse}, we show that the final policy entropy after a fixed number of training steps tracks the positive-to-negative advantage ratio $\rho(p)=|A_+|/|A_-|=\frac{m_S q}{m_F p}$ (Figure~\ref{fig:entropy}). All sign-balanced methods share $\rho(p)=\frac{q}{p}$ since $m_S=m_F$, which is $2.33$ at the average solve rate of $0.3$ that Qwen 2.5 7B reaches midway through training.

\subsection{Full Evaluation Results}
\label{sec:more_results}

We evaluate models trained with different advantage functions at their peak performance during training (typically around $20 \times 10^3$ RL steps), using the same training setup across all runs and varying only the reward normalization. To reduce variance across checkpoints, we select evaluation points where pass@$k$ has stabilized, meaning at least $1000$ consecutive steps yield similar results. For each advantage their required number of training steps can be found in \cref{tab:rl-results-v5}. For each pass@$k$ metric, we draw $2k$ samples and compute the mean over $400$ groups of size $2k$ to estimate the standard deviation. Results can be found for LiveCodeBench v6 in \cref{tab:rl-results-454}, v5 in \cref{tab:rl-results-v5} and for AIME benchmarks in \cref{tab:pass_at_k_7b} and \cref{tab:pass_at_k_32b}.

\begin{table}[ht]
        \centering
        \footnotesize
        \setlength{\tabcolsep}{3pt}
        \renewcommand{\arraystretch}{1.15}
        \begin{tabular}{l ccc cccc}
        \toprule
        & \multicolumn{3}{c}{\textbf{Qwen 2.5 7B SFT}} & \multicolumn{4}{c}{\textbf{CWM 32B}} \\
        \cmidrule(lr){2-4} \cmidrule(lr){5-8}
        & \multicolumn{3}{c}{8k budget} & \multicolumn{4}{c}{30k budget} \\
        \cmidrule(lr){2-4} \cmidrule(lr){5-8}
        \textbf{Method} & p@1 & p@10 & p@100 & p@1 & p@10 & p@100 & p@200 \\ \midrule
        \rowcolor[gray]{0.92} \multicolumn{8}{@{}l}{\textit{Baselines}} \\
        SFT model & 15.2 & 42.3 & 52.2 & 37.5{\scriptsize$\pm$3.4} & 54.0{\scriptsize$\pm$4.0} & 69.0 & 72.5 \\
        REINFORCE & 29.4{\scriptsize$\pm$0.0} & 45.3{\scriptsize$\pm$5.2} & 55.8{\scriptsize$\pm$3.0} & --- & --- & --- & --- \\
        Dr.~GRPO & 37.4{\scriptsize$\pm$2.4} & 53.6{\scriptsize$\pm$2.1} & 62.5{\scriptsize$\pm$0.2} & 62.0{\scriptsize$\pm$1.6} & 76.7{\scriptsize$\pm$0.6} & 83.7 & 85.0 \\
        \midrule
        \rowcolor[gray]{0.92} \multicolumn{8}{@{}l}{\textit{Existing methods}} \\
        Implicit p@2 & 39.7{\scriptsize$\pm$1.7} & 56.2{\scriptsize$\pm$0.8} & 65.2{\scriptsize$\pm$0.9} & --- & --- & --- & --- \\
        Implicit p@8 & 38.0{\scriptsize$\pm$2.0} & 54.3{\scriptsize$\pm$1.6} & 64.7{\scriptsize$\pm$0.6} & --- & --- & --- & --- \\
        LME $\beta=0.4$ & 35.9{\scriptsize$\pm$2.8} & 52.5{\scriptsize$\pm$3.1} & 62.0{\scriptsize$\pm$1.0} & --- & --- & --- & --- \\
        Pass@8 & 33.8{\scriptsize$\pm$2.5} & 51.7{\scriptsize$\pm$3.1} & 61.1{\scriptsize$\pm$1.0} & --- & --- & --- & --- \\
        Asym.\ norm & 38.1{\scriptsize$\pm$1.7} & 53.4{\scriptsize$\pm$1.0} & 63.0{\scriptsize$\pm$0.6} & --- & --- & --- & --- \\
        \midrule
        \rowcolor[gray]{0.92} \multicolumn{8}{@{}l}{\textit{New Methods}} \\
        AsymGRPO $\delta=0.5$ & 35.7{\scriptsize$\pm$2.0} & 51.8{\scriptsize$\pm$1.4} & 60.5{\scriptsize$\pm$0.1} & --- & --- & --- & --- \\
        Power $\alpha=2$ & 40.3{\scriptsize$\pm$1.2} & 56.6{\scriptsize$\pm$0.7} & 66.0{\scriptsize$\pm$0.9} & 61.2{\scriptsize$\pm$1.4} & 76.7{\scriptsize$\pm$0.6} & 82.9 & 83.9 \\
        Power $\alpha$ asym (2,2) & 40.0{\scriptsize$\pm$1.7} & 56.0{\scriptsize$\pm$1.3} & 64.9{\scriptsize$\pm$0.2} & 62.8{\scriptsize$\pm$1.1} & 76.2{\scriptsize$\pm$0.8} & 83.3 & 84.6 \\
        FADE $h^*=1.0$ & \textbf{40.7}{\scriptsize$\pm$1.0} & \textbf{58.0}{\scriptsize$\pm$0.7} & \textbf{68.2}{\scriptsize$\pm$1.2} & \textbf{63.1}{\scriptsize$\pm$1.2} & \textbf{77.9}{\scriptsize$\pm$0.9} &
  \textbf{85.3} & \textbf{86.3} \\
        \bottomrule
        \end{tabular}
        \caption{Pass@$k$ evaluation results with standard deviations on LiveCodeBench v6 (454 problems, Aug 2024-May 2025).}
        \label{tab:rl-results-454}
  \end{table}

\begin{table}[ht]
      \centering
      \footnotesize
      \setlength{\tabcolsep}{3pt}
      \renewcommand{\arraystretch}{1.15}
      \caption{Pass@k evaluation on AIME benchmarks for Qwen 2.5 7B SFT (\%). Models are trained with an 8k token budget on code and evaluated with a 14k budget on math, to test generalization to a new domain and a longer budget.}
      \label{tab:pass_at_k_7b}
      \begin{tabular}{l cccccc}
      \toprule
      & \multicolumn{3}{c}{AIME2024} & \multicolumn{3}{c}{AIME2025} \\
      \cmidrule(lr){2-4} \cmidrule(lr){5-7}
      \textbf{Method} & p@1 & p@10 & p@100 & p@1 & p@10 & p@100 \\
      \midrule
      \rowcolor[gray]{0.92} \multicolumn{7}{@{}l}{\textit{Baselines}} \\
      SFT baseline                   & 32.6 {\tiny$\pm$4.2} & 65.3 {\tiny$\pm$1.2} & 78.1 {\tiny$\pm$4.2} & 21.3 {\tiny$\pm$3.5} & 46.2 {\tiny$\pm$3.0} & 70.8 {\tiny$\pm$5.0} \\
      Dr.~GRPO                           & 37.9 {\tiny$\pm$5.2} & 70.4 {\tiny$\pm$1.9} & 81.4 {\tiny$\pm$2.7} & 24.6 {\tiny$\pm$3.2} & 50.3 {\tiny$\pm$2.5} & 68.3 {\tiny$\pm$3.3} \\
      \midrule
      \rowcolor[gray]{0.92} \multicolumn{7}{@{}l}{\textit{Existing methods}} \\
      Implicit p@2                   & 36.0 {\tiny$\pm$4.2} & 66.1 {\tiny$\pm$2.8} & 73.3 {\tiny$\pm$0.0} & 24.9 {\tiny$\pm$3.0} & 45.0 {\tiny$\pm$3.3} & 63.7 {\tiny$\pm$4.7} \\
      Implicit p@8                   & 40.3 {\tiny$\pm$4.0} & 69.6 {\tiny$\pm$1.9} & 76.7 {\tiny$\pm$1.7} & 27.5 {\tiny$\pm$3.4} & 52.5 {\tiny$\pm$3.3} & 69.1 {\tiny$\pm$5.0} \\
      Asym.\ norm                    & 40.2 {\tiny$\pm$4.2} & 69.1 {\tiny$\pm$1.7} & 78.3 {\tiny$\pm$4.3} & 27.5 {\tiny$\pm$3.4} & 52.0 {\tiny$\pm$2.4} & 63.3 {\tiny$\pm$3.3} \\
      Pass@8                         & 35.7 {\tiny$\pm$4.6} & 68.0 {\tiny$\pm$2.9} & 82.1 {\tiny$\pm$3.2} & 24.4 {\tiny$\pm$3.0} & 49.3 {\tiny$\pm$2.4} & 70.4 {\tiny$\pm$6.9} \\
      LME $\beta=0.4$                & 39.3 {\tiny$\pm$4.7} & 71.0 {\tiny$\pm$3.4} & 82.4 {\tiny$\pm$1.9} & 25.4 {\tiny$\pm$3.1} & 51.0 {\tiny$\pm$4.0} & \textbf{72.5 {\tiny$\pm$6.9}} \\
      \midrule
      \rowcolor[gray]{0.92} \multicolumn{7}{@{}l}{\textit{New methods}} \\
      AsymGRPO $\delta=0.5$       & 43.1 {\tiny$\pm$4.6} & 69.1 {\tiny$\pm$1.0} & 79.2 {\tiny$\pm$2.7} & 26.6 {\tiny$\pm$3.0} & 48.3 {\tiny$\pm$3.2} & 66.0 {\tiny$\pm$3.2} \\
      Power $\alpha=2$                  & 41.8 {\tiny$\pm$4.4} & 70.4 {\tiny$\pm$1.3} & 77.9 {\tiny$\pm$1.7} & 26.6 {\tiny$\pm$2.8} & 50.0 {\tiny$\pm$3.0} & 64.9 {\tiny$\pm$1.7} \\
      Power $\alpha$ asym (2,2)                  & 41.1 {\tiny$\pm$4.2} & 71.9 {\tiny$\pm$1.9} & \textbf{86.1 {\tiny$\pm$5.0}} & 27.4 {\tiny$\pm$3.4} & 53.3 {\tiny$\pm$2.2} & 66.6 {\tiny$\pm$2.7} \\
      \rowcolor[gray]{0.92} \multicolumn{7}{@{}l}{\textit{FADE (ours)}} \\
      FADE $h^*=0.5$                 & 41.9 {\tiny$\pm$4.2} & 70.9 {\tiny$\pm$1.1} & 79.6 {\tiny$\pm$3.2} & 27.8 {\tiny$\pm$2.7} & 52.9 {\tiny$\pm$2.2} & 67.5 {\tiny$\pm$4.3} \\
      FADE $h^*=1.0$                 & \textbf{44.5 {\tiny$\pm$4.6}} & \textbf{72.3 {\tiny$\pm$2.2}} & 83.7 {\tiny$\pm$5.0} & 28.6 {\tiny$\pm$3.4} & 53.8 {\tiny$\pm$1.8} & 68.3 {\tiny$\pm$1.9} \\
      FADE $h^*=1.3$                 & 44.3 {\tiny$\pm$3.6} & 71.8 {\tiny$\pm$2.0} & 80.8 {\tiny$\pm$0.0} & \textbf{31.3 {\tiny$\pm$3.7}} & \textbf{55.4 {\tiny$\pm$2.9}} & 70.7 {\tiny$\pm$1.9} \\
      \bottomrule
      \end{tabular}
\end{table}

\begin{table}[ht]
      \centering
      \footnotesize
      \setlength{\tabcolsep}{3pt}
      \renewcommand{\arraystretch}{1.15}
      \caption{Pass@k evaluation on AIME benchmarks for CWM 32B (30k training budget) (\%).}
      \label{tab:pass_at_k_32b}
      \begin{tabular}{l cccccc}
      \toprule
      & \multicolumn{3}{c}{AIME2024} & \multicolumn{3}{c}{AIME2025} \\
      \cmidrule(lr){2-4} \cmidrule(lr){5-7}
      \textbf{Method} & p@1 & p@10 & p@100 & p@1 & p@10 & p@100 \\
      \midrule
      \rowcolor[gray]{0.92} \multicolumn{7}{@{}l}{\textit{Baselines}} \\
      Dr.~GRPO                           & 58.4 {\tiny$\pm$5.0} & 87.4 {\tiny$\pm$0.8} & \textbf{94.9 {\tiny$\pm$1.7}} & 46.1 {\tiny$\pm$4.8} & 77.2 {\tiny$\pm$2.0} & 80.0 {\tiny$\pm$0.0} \\
      \midrule
      \rowcolor[gray]{0.92} \multicolumn{7}{@{}l}{\textit{New methods (ours)}} \\
      Power $\alpha=2$                  & 41.2 {\tiny$\pm$5.3} & 81.9 {\tiny$\pm$2.6} & 91.7 {\tiny$\pm$1.7} & 27.4 {\tiny$\pm$4.6} & 70.4 {\tiny$\pm$2.9} & \textbf{84.6 {\tiny$\pm$3.2}} \\
      Power asym (2,2)                  & \textbf{70.3 {\tiny$\pm$3.5}} & 89.2 {\tiny$\pm$1.5} & 93.3 {\tiny$\pm$1.7} & \textbf{58.7 {\tiny$\pm$4.2}} & 78.1 {\tiny$\pm$1.4} & 83.3 {\tiny$\pm$3.2} \\
      Asym FADE $h^*=1.0$            & 64.6 {\tiny$\pm$4.7} & \textbf{89.5 {\tiny$\pm$2.6}} & 93.3 & 52.7 {\tiny$\pm$5.1} & \textbf{78.5 {\tiny$\pm$1.6}} & 83.1 \\
      \bottomrule
      \end{tabular}
\end{table}

\begin{table}[ht]
    \centering
    \footnotesize
    \setlength{\tabcolsep}{3pt}
    \renewcommand{\arraystretch}{1.15}
    \begin{tabular}{l cccc ccccc}
    \toprule
    & \multicolumn{4}{c}{\textbf{Qwen 2.5 7B SFT}} & \multicolumn{5}{c}{\textbf{CWM 32B}} \\
    \cmidrule(lr){2-5} \cmidrule(lr){6-10}
    & \multicolumn{4}{c}{8k budget} & \multicolumn{5}{c}{30k budget} \\
    \cmidrule(lr){2-5} \cmidrule(lr){6-10}
    \textbf{Method} & Steps & p@1 & p@10 & p@100 & Steps & p@1 & p@10 & p@100 & p@200 \\ \midrule
    \rowcolor[gray]{0.92} \multicolumn{10}{@{}l}{\textit{Baselines}} \\
    SFT model & 0 & 21.9{\scriptsize$\pm$0.8} & 59.7{\scriptsize$\pm$0.4} & 71.2{\scriptsize$\pm$0.2} & 0 & 51.6{\scriptsize$\pm$0.8} & 75.6{\scriptsize$\pm$0.4} & 85.1 & 86.9
\\
    REINFORCE & 17.5k & 40.7{\scriptsize$\pm$0.6} & 60.2{\scriptsize$\pm$0.4} & 69.8{\scriptsize$\pm$0.2} & --- & --- & --- & --- & --- \\
    Dr.~GRPO & 20k & 52.6{\scriptsize$\pm$0.7} & 71.8{\scriptsize$\pm$0.4} & 80.6{\scriptsize$\pm$0.2} & 13.4k & 77.3{\scriptsize$\pm$0.6} & 89.0{\scriptsize$\pm$0.2} & 92.0 & 92.6
\\
    \midrule
    \rowcolor[gray]{0.92} \multicolumn{10}{@{}l}{\textit{Existing methods}} \\
    Implicit p@2 & 30k & 55.3{\scriptsize$\pm$0.7} & 73.3{\scriptsize$\pm$0.4} & 80.9{\scriptsize$\pm$0.2} & --- & --- & --- & --- & --- \\
    Implicit p@8 & 30k & 55.7{\scriptsize$\pm$0.6} & 72.5{\scriptsize$\pm$0.4} & 80.9{\scriptsize$\pm$0.2} & --- & --- & --- & --- & --- \\
    LME $\beta=0.4$ & 20k & 51.0{\scriptsize$\pm$0.6} & 70.8{\scriptsize$\pm$0.4} & 79.9{\scriptsize$\pm$0.2} & --- & --- & --- & --- & --- \\
    Pass@8 & 30k & 48.3{\scriptsize$\pm$0.7} & 70.5{\scriptsize$\pm$0.4} & 80.0{\scriptsize$\pm$0.2} & --- & --- & --- & --- & --- \\
    Asym.\ norm & 20k & 54.1{\scriptsize$\pm$0.6} & 70.9{\scriptsize$\pm$0.4} & 79.2{\scriptsize$\pm$0.2} & --- & --- & --- & --- & --- \\
    \midrule
    \rowcolor[gray]{0.92} \multicolumn{10}{@{}l}{\textit{New Methods}} \\
    AsymGRPO $\delta=0.5$ & 40k & 53.4{\scriptsize$\pm$0.6} & 71.1{\scriptsize$\pm$0.4} & 79.6{\scriptsize$\pm$0.2} & --- & --- & --- & --- & --- \\
    Power $\alpha=2$ & 50k & \textbf{56.3}{\scriptsize$\pm$0.6} & \textbf{73.8}{\scriptsize$\pm$0.4} & \textbf{82.1}{\scriptsize$\pm$0.2} & 17.5k & 77.9{\scriptsize$\pm$0.6} &
\textbf{89.4}{\scriptsize$\pm$0.2} & 92.3 & 92.7 \\
    Power $\alpha$ asym (2,2) & 20k & 56.1{\scriptsize$\pm$0.6} & 73.6{\scriptsize$\pm$0.4} & 81.5{\scriptsize$\pm$0.2} & 15k & 78.7{\scriptsize$\pm$0.5} &
88.9{\scriptsize$\pm$0.2} & 92.0 & 92.6 \\
    FADE $h^*=1.0$ & 30k & 56.1{\scriptsize$\pm$0.6} & 73.6{\scriptsize$\pm$0.3} & 82.0{\scriptsize$\pm$0.2} & 15k & \textbf{78.9}{\scriptsize$\pm$0.5} &
89.3{\scriptsize$\pm$0.2} & \textbf{92.5} & \textbf{93.2} \\
    \bottomrule
    \end{tabular}
    \caption{Pass@$k$ evaluation results with standard deviations on LiveCodeBench v5 (879 problems: 279 easy, 330 medium, 270 hard). 7B at 8k token budget, CWM 32B at 30k. Steps is the RL step selected for evaluation when evaluation results plateau (SFT $=0$, i.e.\ no RL).}
    \label{tab:rl-results-v5}
\end{table}

\clearpage
\section{Deriving Positive and Negative Weights}
\label{sec:weight_derivations}
Throughout, $\tau_S$ and $\tau_F$ denote successful and failed trajectories (positive and negative weights), $\bar{r} := \frac{1}{B}\sum_i r_i$ the batch reward mean, $\sigma_r = \sqrt{\frac{1}{B}\sum_i(r_i - \bar{r})^2}$ the reward standard deviation, and all gradients are estimated over $B$ independent rollouts from~$\pi_{\theta}$.

\paragraph{How to read the masses.} By Section~\ref{sec:knobs}, a row with masses $(m_S,m_F)$ has, in population, prompt profile $c(p)=m_S/p+m_F/q$ and implied baseline $b(p)=m_F/(q\,c(p))$, i.e.\ $A\propto c\,(r-b)$. Only $c$ changes the expected update; $b\neq p$ (``sign imbalance'') changes its variance and gives each group a non-zero net advantage (Section~\ref{sec:knobs}). These are the columns reported in Table~\ref{tab:weights}.

\subsection{GRPO}
\label{sec:grpo_derivation}
The GRPO advantage is $A_i = (r_i - \bar{r})/\sigma_r$. For binary rewards $r_i \in \{0,1\}$ with $\bar{r} = \hat{p}$, the empirical standard deviation reduces to the Bernoulli form $\sigma_r = \sqrt{\hat{p}(1-\hat{p})} = \sqrt{pq}$, since:
\[ \sigma_r^2 = \tfrac{1}{B}\textstyle\sum_{i=1}^B (r_i - \hat{p})^2 = \tfrac{1}{B}\bigl[Bp\,(1{-}p)^2 + Bq\, p^2\bigr] = pq(q + p) = pq. \]
The per-trajectory advantages are:
\[ A_s = \frac{q}{\sqrt{pq}} = \sqrt{\frac{q}{p}}, \qquad A_f = \frac{-p}{\sqrt{pq}} = -\sqrt{\frac{p}{q}}.\]
With $|S| = Bp$ and $|F| = Bq$ in expectation:
\begin{align*}
\nabla_{\mathrm{GRPO}} &= \tfrac{q}{\sqrt{pq}} \cdot Bp \cdot \mS - \tfrac{p}{\sqrt{pq}} \cdot Bq \cdot \mF = \tfrac{Bpq}{\sqrt{pq}} \left[\mS - \mF\right] = B\sqrt{pq}\left[\mS - \mF\right].
\end{align*}
Hence $m_S = m_F = pq/\sigma_r = \sqrt{pq}$, confirming sign balance.
Skew-R introduced by \citet{thrampoulidis2025advantage} \citep[see also][]{team2025kimi} multiplies the Dr.~GRPO advantage by the reward standard deviation $\sigma_r=\sqrt{pq}$ instead of dividing by it, giving $m_S = m_F = pq \times \sqrt{pq}$.

\subsection{Multiplier Rescaled GRPO}
The $\frac{1}{\sigma_r}$ multiplier in GRPO is unstable \citep{liu2025understanding} \citep{yu2025dapo}, as it can lead to exploding or noisy gradient updates. Instead, we can keep the mean normalization and introduce other multipliers:
\begin{align*}
    \nabla_{\text{rescale}} &= \sum_{i \in S} C_s(\mathbf{r}, i) \cdot (r_i - \bar{r}) \nabla \log \pi_{\theta}(\tau_i) - \sum_{i \in F} C_f(\mathbf{r}, i) \cdot (\bar{r} - r_i) \nabla \log \pi_{\theta}(\tau_i),
\end{align*}
where $C_s(\mathbf{r}, i)$ and $C_f(\mathbf{r}, i)$ are arbitrary non-negative functions that may depend on the full reward vector $\mathbf{r} = (r_1, \dots, r_B)$ and the sample index $i$. Mirroring the logic behind the focal loss \citep{lin2017focal}, we introduce the Power $\alpha$ series where the functions are powers of $|r_i - \bar{r}|$.

Our proposed methods:
\begin{enumerate}
    \item \textbf{Power $\alpha$ (sign-balanced):} The intuition is to reduce the gradient magnitude when the model already performs well ($\bar{r}$ high): samples that are easy to solve should contribute less to the update. We set $C_s(\mathbf{r}, i) = C_f(\mathbf{r}, i) = (1 - \bar{r})^{\alpha - 1}$ for all $i$, so the advantage becomes:
    \[\hat{A}_i = (r_i - \bar{r})(1 - \bar{r})^{\alpha - 1}.\]
    For binary rewards where $\bar{r} = p$, this simplifies to $C_s = C_f = q^{\alpha - 1}$, giving:
    \[m_S = m_F = pq \cdot q^{\alpha - 1} = pq^{\alpha}.\]
    As $\alpha$ increases, the masses decay faster with $q$: gradients are suppressed on prompts the model has mostly solved, focusing optimization on harder problems. In practice we multiply the advantage by $C_\alpha=\frac{(1+\alpha)^{1+\alpha}}{4\,\alpha^{\alpha}}$, so that the peak of $C_\alpha\,pq^{\alpha}$, reached at $p=\frac{1}{1+\alpha}$, equals the Dr.~GRPO peak $\max_p pq=\frac14$. Table~\ref{tab:weights} drops this constant.

    \item \textbf{Asymmetric Power $\alpha$:} We decouple the exponents: $C_s(\mathbf{r}, i) = (1 - \bar{r})^{\alpha_s - 1}$ for $i \in S$ and $C_f(\mathbf{r}, i) = \bar{r}^{\alpha_f - 1}$ for $i \in F$. The positive mass derivation is identical to the sign-balanced case:
    \[ m_S = pq^{\alpha_s}.\]
    For the negative side, with binary rewards $\bar{r} = p$:
    \[ m_F = p^{\alpha_f - 1} \cdot p \cdot q = p^{\alpha_f}q.\]
    If $\alpha_s > \alpha_f$ we suppress the positive gradient more than the negative one (or vice versa), giving explicit control over the diversity--performance tradeoff. Note that $\alpha_s = \alpha_f = 1$ recovers Dr.~GRPO ($m_S = m_F = pq$).

    \item \textbf{AsymGRPO:} keeps the Dr.~GRPO advantage intact on the positive side and rescales only the negative side by a fixed scalar $\frac{1}{\delta}$: $C_s = 1$ and $C_f = \frac{1}{\delta}$. This gives:
    \[ m_S = pq, \quad m_F = \frac{pq}{\delta}.\]
    When $\delta < 1$, the negative gradient is amplified relative to the positive one ($m_F > m_S$), pushing the model more aggressively away from failed trajectories. When $\delta > 1$, the negative gradient is suppressed ($m_F < m_S$), so the update is dominated by reinforcing successes, which accelerates entropy collapse (Section~\ref{sec:entropy_collapse}). At $\delta = 1$ we recover Dr.~GRPO. AsymGRPO splits as $A_i=(\hat q+\hat p/\delta)(r_i-\hat p)+\hat p\hat q(1-1/\delta)$, a sign-balanced advantage plus a constant per-group shift, so it isolates the effect of sign bias. Unlike Power~$\alpha$, this method does not adapt to the difficulty of the prompt ($p$). It modifies only the strength of negative advantages.
\end{enumerate}

\subsection{Maximum Likelihood RL}

As introduced in \citet{tajwar2026maximumlikelihoodreinforcementlearning}, the MaxRL advantage is $A = \frac{r_i - \bar{r}}{\bar{r}}$. With $r \in \{0,1\}$, we have $A_s = \frac{1-p}{p}$ and $A_f = \frac{0-p}{p} = -1$. Similar to the GRPO gradient we can write:
\begin{align*}
\nabla_{\mathrm{MaxRL}} &= \sum_{\tau_s \in S} A_s \nabla \log \pi_{\theta}(\tau_s) + \sum_{\tau_f \in F} A_f \nabla_{\theta} \log \pi_{\theta}(\tau_f) \\
&= \frac{1-p}{p} \sum_{\tau_s \in S} \nabla \log \pi_\theta(\tau_S) - \sum_{\tau_f \in F} \nabla \log \pi_\theta(\tau_F) \\
&= \frac{(1-p)}{p} \cdot Bp \cdot \mS - B(1-p) \cdot \mF = B(1-p) \left[ \mS - \mF \right].
\end{align*}

\paragraph{Finite-$G$ form.}\label{app:maxrl_finite}
With the self-inclusive mean $\bar r=\hat p$, the expected MaxRL update has a closed form at finite $G$.

\begin{proposition}[MaxRL at finite $G$]
For $A_i=(r_i-\hat p)/\hat p$, with $A_i=0$ when $\hat p=0$,
\[
\mathbb E\Big[\sum_{i=1}^G A_i\,s_i\Big]=G\,c_G(p)\,\nabla p,
\qquad
c_G(p)=\frac{1-q^{G-1}}{p}=\sum_{k=1}^{G-1}q^{k-1}.
\]
Hence MaxRL with $G$ rollouts is the unbiased gradient of $\tau_G(p)=\sum_{k=1}^{G-1}\mathrm{pass@}k/k$, with $\mathrm{pass@}k=1-q^k$. Up to the constant $\sum_{k<G}1/k$, this is the order-$(G{-}1)$ truncation of $\log p=-\sum_{k\ge1}q^k/k$.
\end{proposition}

\begin{proof}
Fix rollout $i$ and let $K'\sim\mathrm{Bin}(G-1,p)$ count the successes among the other rollouts. A success has $\hat p=(K'+1)/G$, so $A_i=w_S=(G-1-K')/(K'+1)$. A failure has $A_i=-w_F$ with $w_F=\mathbf 1[K'>0]$. Given the rewards, the rollouts are independent, so $\mathbb E[s_i\mid r_1,\dots,r_G]=\mu_{r_i}$, and $p\,\mu_S=-q\,\mu_F=\nabla p$ gives
\[
\mathbb E[A_is_i]=p\,\mathbb E[w_S]\,\mu_S-q\,\mathbb E[w_F]\,\mu_F=\mathbb E[w_S+w_F]\,\nabla p .
\]
Using $\mathbb E[1/(K'+1)]=(1-q^G)/(Gp)$,
\[
c_G(p)=\mathbb E[w_S+w_F]=\mathbb E\Big[\frac{G}{K'+1}\Big]-1+\Pr[K'>0]=\frac{1-q^G}{p}-q^{G-1}=\frac{1-q^{G-1}}{p}.
\]
Summing over the $G$ rollouts gives the update, and $\partial_p(1-q^k)/k=q^{k-1}$ gives the objective.
\end{proof}
Exact enumeration over $K'$ and Monte Carlo confirm $c_G$ for $G\in\{4,8,16\}$ and $p\in[0.02,0.8]$. The profile is $\approx1/p$ only for $p\gg1/G$ and saturates at $c_G(0)=G-1$. On $p\in[0.05,0.5]$, matching $c_G(0.05)/c_G(0.5)$ to $q^{\alpha-1}$ gives Power $\alpha\approx2.7$ at $G=8$ and $\alpha\approx3.6$ at $G=16$. At practical group sizes, MaxRL is thus close to Power $\alpha\approx3$, and it approaches $\log p$ only as $G$ grows.

\subsection{Pass@$k$ Based Methods}

\subsubsection{Pass@$k$ (Tang et al.)}
The leave-one-out advantage of \citet{tang2025optimizing} is $A_i = \max_k(\mathbf{r}) - \max_k(\mathbf{r}_{-i})$, where $\max_k(\mathbf{r}) = \mathbf{1}[\exists\, j : r_j = 1]$ is the pass@$k$ indicator over a group of $k$ samples and $\mathbf{r}_{-i}$ excludes sample $i$.

\paragraph{Failed samples ($r_i = 0$).} Removing a zero does not change the maximum: $\max_k(\mathbf{r}) = \max_k(\mathbf{r}_{-i})$ regardless of the other samples, so $A_i = 0$ for all $i \in F$.

\paragraph{Successful samples ($r_i = 1$).} We have $\max_k(\mathbf{r}) = 1$. The leave-one-out is $\max_k(\mathbf{r}_{-i}) = \mathbf{1}[\exists\, j \neq i : r_j = 1]$, so:
\[ A_i = 1 - \mathbf{1}[\exists\, j \neq i : r_j = 1] = \mathbf{1}[\text{sample } i \text{ is the only success}].\]
Hence $A_i = 1$ when $|S| = 1$ and $i$ is the unique success, and $A_i = 0$ when $|S| \geq 2$.

\paragraph{Gradient.} Since only the unique-success case contributes:
\begin{align*}
\nabla &= \sum_{i=1}^{k} A_i \nabla \log \pi_\theta(\tau_i) = \begin{cases} \nabla \log \pi_\theta(\tau_s) = \mS & \text{if } |S| = 1, \\[4pt] 0 & \text{otherwise.}\end{cases}
\end{align*}
Taking expectations, $\Pr[|S| = 1] = \binom{k}{1} p\, q^{k-1} = k p q^{k-1}$, so:
\[ \mathbb{E}[\nabla] = k p q^{k-1}\, \mS - 0 \cdot \mF,\]
giving $m_S = k p q^{k-1}$ and $m_F = 0$. These are per-group totals; the per-sample masses used for the other methods are $k$ times smaller, $m_S=pq^{k-1}$.

Unlike the sign-balanced methods (where the masses are deterministic for any fixed batch), here the gradient is zero in most batches and only fires when exactly one of $k$ samples succeeds.

\subsubsection{Analytical Pass@$k$ (Chen et al.)}
\citet{chen2025pass} define the group pass@$k$ estimator $R_k = 1 - \binom{F}{k}\!/\!\binom{G}{k}$ over $G$ rollouts with $F$ failures and $S = G - F$ successes, with $\sigma_k = \sqrt{R_k(1 - R_k)}$. The per-sample advantages are:
\[ A_s = \frac{1 - R_k}{\sigma_k}, \qquad A_f = \frac{1}{\sigma_k}\!\left(1 - R_k - \frac{\binom{F-1}{k-1}}{\binom{G-1}{k-1}}\right).\]

\paragraph{Simplification.} Define $Q_{k-1}' := \binom{F{-}1}{k{-}1}\!/\!\binom{G{-}1}{k{-}1}$. The identity $\binom{F}{k}\!/\!\binom{G}{k} = (F/G)\,Q_{k-1}'$ gives $1 - R_k = \hat{q}\, Q_{k-1}'$, so:
\[ A_s = \frac{\hat{q}\, Q_{k-1}'}{\sigma_k}, \qquad A_f = \frac{Q_{k-1}'(\hat{q} - 1)}{\sigma_k} = \frac{-\hat{p}\, Q_{k-1}'}{\sigma_k}.\]
Both cases unify into $A_i = Q_{k-1}'(r_i - \hat{p})/\sigma_k$, which is exact for any finite $G$ and has the $C \cdot (r_i - \hat{p})$ structure guaranteeing sign balance. Following the GRPO pattern with $|S| = G\hat{p}$ and $|F| = G\hat{q}$:
\begin{align*}
\nabla &= \frac{\hat{q}\,Q_{k-1}'}{\sigma_k} \cdot S \cdot \mS - \frac{\hat{p}\,Q_{k-1}'}{\sigma_k} \cdot F \cdot \mF = \frac{G\hat{p}\hat{q}\,Q_{k-1}'}{\sigma_k}\left[\mS - \mF\right].
\end{align*}
Substituting $Q_{k-1}' = (1-R_k)/\hat{q}$ and $\sigma_k = \sqrt{R_k(1-R_k)}$:
\[ m_S = m_F = \hat{p}\,\frac{1-R_k}{\sigma_k} = \hat{p}\sqrt{\frac{1-R_k}{R_k}}.\]
At $k = 1$, $R_1 = S/G = \hat{p}$, so $m_S = m_F = \sqrt{\hat{p}\hat{q}}$, recovering GRPO. In the population limit ($G \to \infty$), $R_k \to 1 - q^k$ and $Q_{k-1}' \to q^{k-1}$, giving $A_i \to q^{k-1}(r_i - p)/\sigma_k$ and $m_S = m_F = p\sqrt{q^k/(1-q^k)}$.

\textbf{Remark.} If one were to naively set the baseline to $R_k$ instead of $\hat{p}$, i.e.\ $A_i = (r_i - R_k)/\sigma_k$, one would obtain $m_S = p(1{-}R_k)/\sigma_k \neq m_F = qR_k/\sigma_k$ for $k > 1$. The per-sample baseline must remain at $\hat{p}$ for sign symmetry to hold.

\subsubsection{Mix Pass@1/Pass@$k$ (Chen et al.)}
\citet{chen2025pass} also propose a convex combination of pass@1 and pass@$k$ advantages:
\[ A_i^{mix} = p \cdot \hat{A}_{pass@k} + q \cdot \hat{A}_{pass@1}.\]
Both components have the mean-centered structure $C \cdot (r_i - p)$:
\[ \hat{A}_{pass@1} = \frac{r_i - p}{\sigma_1}, \qquad \hat{A}_{pass@k} = \frac{q^{k-1}(r_i - p)}{\sigma_k}.\]
A linear combination of mean-centered terms remains mean-centered:
\[ A_i^{mix} = \left(\frac{p q^{k-1}}{\sigma_k} + \frac{q}{\sigma_1}\right)(r_i - p) =: C_{mix}(r_i - p).\]
Since this is of the form $C \cdot (r_i - p)$, it is automatically sign-balanced. Following the GRPO pattern:
\begin{align*}
\nabla &= C_{mix}\left[q \sum_{i \in S} \nabla \log \pi_\theta(\tau_i) \right. \left. - p \sum_{i \in F} \nabla \log \pi_\theta(\tau_i)\right] = C_{mix} \cdot B p q \left[\mS - \mF\right].
\end{align*}
So $m_S = m_F = C_{mix} \cdot pq$. Noting that the success-side advantages are $\hat{A}_{s}^{(k)} = C_k q = q^k / \sigma_k$ and $\hat{A}_{s}^{(1)} = q / \sigma_1$:
\begin{align*}
    C_{mix} \cdot pq &= p\!\left(p \cdot \underbrace{C_k q}_{\hat{A}_{s}^{(k)}} + q \cdot \underbrace{C_1 q}_{\hat{A}_{s}^{(1)}}\right) = p\!\left(p\, \hat{A}_{pass@k} + q\, \hat{A}_{pass@1}\right),
\end{align*}
confirming $m_S = m_F = p\!\left(p\, \hat{A}_{pass@k} + q\, \hat{A}_{pass@1}\right)$.

\subsection{Biasing towards Negatives with a Shifted Baseline}
\label{sec:asymrl_derivation}

A series of policy weights propose to shift up or down the positive weights by adding a fixed offset to Dr.~GRPO \citep{arnal2026asymmetric} or to REINFORCE \citep{Williams1992}, or clipping the baseline from below at a minimum acceptable reward \citep{CorPO}. We'll derive the advantage weight of AsymRL which can be generalized to other types of offsets, the advantage of AsymRL is:
\[ A_i = r_i - (\bar{r} + \delta_A).\]
With binary rewards $r_i \in \{0,1\}$ and $\bar{r} = p$, the per-trajectory weights are:
\[ A_s = 1 - p - \delta_A = q - \delta_A, \qquad A_f = -p - \delta_A = -(p + \delta_A).\]
Splitting the gradient over successful and failed trajectories:
\begin{align*}
\nabla_{\mathrm{AsymRL}} &= \sum_{i \in S} (q - \delta_A)\, \nabla \log \pi_\theta(\tau_i) - \sum_{i \in F} (p + \delta_A)\, \nabla \log \pi_\theta(\tau_i) \\
&= (q - \delta_A) \cdot |S| \cdot \mS - (p + \delta_A) \cdot |F| \cdot \mF \\
&= B p (q - \delta_A)\, \mS - B q (p + \delta_A)\, \mF,
\end{align*}
At $\delta_A = 0$ we recover Dr.~GRPO but in practice, \citet{arnal2026asymmetric} recommend $\delta_A = 0.01$.

CoRPO~\citep{CorPO} instead raises the baseline only when the group falls below a minimum acceptable reward $r_{\min}$, so that failures are never reinforced: $A_i=r_i-\max(r_{\min},\bar r)$. With binary rewards, $A_s=1-\max(r_{\min},p)$ and $A_f=-\max(r_{\min},p)$, so
\[ m_S=p\,\big(1-\max(r_{\min},p)\big),\qquad m_F=q\max(r_{\min},p). \]
For $p\ge r_{\min}$ this is Dr.~GRPO; for $p<r_{\min}$ the baseline exceeds $p$ and the update is failure-biased, with $c\equiv1$ throughout.

\subsection{Quantile Baseline and MC-GRPO}
\label{sec:quantile_derivation}

The quantile baseline~\citep{dabney2018distributional} replaces the mean baseline with the $\kappa$-quantile of the batch rewards: $A_i = r_i - Q_\kappa(\mathbf{r})$, where $Q_\kappa$ is the $\kappa$-th quantile. With binary rewards $r_i \in \{0,1\}$, the quantile is a step function of the solve rate:
\[ Q_\kappa(\mathbf{r}) = \begin{cases} 0 & \text{if } p \le 1-\kappa, \\ 1 & \text{if } p > 1-\kappa. \end{cases} \]

\paragraph{Case 1: $p \le 1-\kappa$ (hard problems).} The baseline is $0$, so $A_s = 1$, $A_f = 0$. Only successes contribute:
\[ \nabla = \sum_{i \in S} \nabla \log \pi_\theta(\tau_i) = Bp\, \bar{\nabla}_S. \]
Hence $m_S = p$ and $m_F = 0$.

\paragraph{Case 2: $p > 1-\kappa$ (easy problems).} The baseline is $1$, so $A_s = 0$, $A_f = -1$. Only failures contribute:
\[ \nabla = -\sum_{i \in F} \nabla \log \pi_\theta(\tau_i) = -Bq\, \bar{\nabla}_F. \]
Hence $m_S = 0$ and $m_F = q$.

Combining both cases: $m_S = p\,\mathbf{1}[p \le 1-\kappa]$ and $m_F = q\,\mathbf{1}[p > 1-\kappa]$, i.e.\ $c\equiv 1$ (the pass@1 direction) with a baseline that jumps from $0$ to $1$. The quantile baseline is always sign-biased for binary rewards: it learns only from successes on hard problems and only from failures on easy ones, never both simultaneously.

MC-GRPO~\citep{kim2026mc} uses the median as the baseline, $A_i = r_i - \mathrm{med}(\mathbf{r})$, which is the special case $\kappa = \tfrac{1}{2}$. With binary rewards, $\mathrm{med}(\mathbf{r}) = \mathbf{1}[p > \tfrac{1}{2}]$, so the masses reduce to $m_S = p\,\mathbf{1}[p \le \tfrac{1}{2}]$ and $m_F = q\,\mathbf{1}[p > \tfrac{1}{2}]$: the method switches abruptly from pure positive reinforcement to pure negative reinforcement at $p = \tfrac{1}{2}$.

\subsection{Binary Contrastive}
\label{sec:binary_contrastive_derivation}
The binary contrastive weight~\citep{greensmith2004variance} is $W(\tau) = \mathbf{1}_{r=1} - \frac{\bar{p}}{1-\bar{p}}\,\mathbf{1}_{r=0}$ so we have $m_S = p \times 1$ and $m_F = (1-p) \times \frac{p}{1-p} = p$.

\subsection{Power Norm}
\label{sec:powernorm_derivation}
The Power Norm advantage~\citep{andrychowicz2020matters} is $A_i = \frac{r_i - \bar{r}}{[\bar{r}(1-\bar{r})]^\gamma}$. With binary rewards and $\bar{r} = p$:
\[ A_s = \frac{q}{(pq)^\gamma}, \qquad A_f = -\frac{p}{(pq)^\gamma}.\]
The gradient is:
\begin{align*}
\nabla &= \frac{q}{(pq)^\gamma} \cdot Bp\,\mS - \frac{p}{(pq)^\gamma} \cdot Bq\,\mF = \frac{Bpq}{(pq)^\gamma}\left[\mS - \mF\right].
\end{align*}
Hence $m_S = m_F = pq/(pq)^\gamma = p^{1-\gamma}q^{1-\gamma}$. At $\gamma = 1/2$ this recovers GRPO; at $\gamma = 0$ it recovers Dr.\ GRPO.

\subsection{Function-Based Policy Weights}
\label{sec:softmax_derivation}

Both Softmax and Logmeanexp aggregate rewards through the exponential partition function. With $|S| = Gp$ successes and $|F| = Gq$ failures in a batch of $G$ rollouts, the partition function is:
\[ Z := |S|\,e^\beta + |F| = G(pe^\beta + q) = G\bigl[1 + p(e^\beta - 1)\bigr].\]
The per-sample softmax weights are $\mathrm{softmax}_\beta(1) = e^\beta/Z$ and $\mathrm{softmax}_\beta(0) = 1/Z$.

\paragraph{Softmax~\citep{deepseek_grpo}.} The advantage $A_i = \mathrm{softmax}_\beta(r_i) - 1/G$ subtracts a uniform baseline. Substituting:
\[ A_s = \frac{e^\beta}{Z} - \frac{1}{G} = \frac{(e^\beta - 1)q}{Z}, \qquad A_f = \frac{1}{Z} - \frac{1}{G} = \frac{-p(e^\beta - 1)}{Z}.\]
The gradient is:
\begin{align*}
\nabla &= A_s \cdot Gp\,\mS - |A_f| \cdot Gq\,\mF = \frac{(e^\beta - 1)pq}{pe^\beta + q}\left[\mS - \mF\right].
\end{align*}
Hence $m_S = m_F = \frac{(e^\beta - 1)pq}{G\,(1 + p(e^\beta - 1))}$ per sample, which is sign-balanced. As $\beta \to 0$ it is Dr.\ GRPO scaled by $\beta/G$; as $\beta \to \infty$ it converges to MaxRL$/G$ (prompt profile $c\propto 1/p$), not to a pass@$k$ weighting.

\paragraph{Logmeanexp~\citep{lme}.} \label{sec:logmeanexp_derivation} The advantage $A_i = \mathrm{lme}_\beta(\mathbf{r}) - \mathrm{lme}_\beta(\mathbf{r}_{-i})$ uses the same partition function through a leave-one-out log-difference, where $\mathrm{lme}_\beta(\mathbf{r}) = \frac{1}{\beta}\log(Z/G)$. Removing a success shifts $Z \to Z - e^\beta$ while removing a failure shifts $Z \to Z - 1$. Expanding the leave-one-out difference to first order in $1/G$ gives a sign-balanced weight,
\[ m_S = m_F \approx \frac{(e^\beta - 1)\,pq}{\beta G\,(pe^\beta + q)}, \]
i.e.\ prompt profile $c\propto 1/(pe^\beta+q)$ (easy-leaning) and baseline $b=p$. Exact leave-one-out computation at $G=16$ and $\beta=0.4$ confirms $m_S/m_F\in[1.012,1.016]$ for $p\in[0.1,0.7]$. As $\beta\to 0$ it reduces to RLOO$/G$.

\subsection{HA-DW}
\label{sec:hadw_derivation}
HA-DW~\citep{ha_dw} uses a hardness-aware dynamic weight: $W(\tau) = A \cdot \lambda \exp(-\mathrm{sgn}(\hat{A})\,\mathrm{sgn}(\hat{r} - C_t)\,|\hat{r} - C_t|)$, where $C_t$ is a running baseline. With binary rewards and base advantage $A=r-p$, the exponent is $-(p-C_t)$ for successes ($\hat A>0$) and $+(p-C_t)$ for failures ($\hat A<0$), so both signs are modulated:
\begin{align*}
m_S &= \lambda e^{-(p-C_t)}\,pq, \qquad m_F = \lambda e^{p-C_t}\,pq.
\end{align*}
The prompt weight is $c=\lambda\big(q\,e^{-(p-C_t)}+p\,e^{p-C_t}\big)$ and the implied baseline is $b=m_F/(qc)=p/\big(p+q\,e^{-2(p-C_t)}\big)$. Prompts solved more often than the running baseline ($p>C_t$) are failure-biased, and prompts solved less often are success-biased.

\subsection{ReLU Advantage}
\label{sec:relu_derivation}
The ReLU advantage~\citep{srinivasan2018actor} is $A_i = \max(0, r_i - \bar{r})$. With binary rewards: $A_s = \max(0, q) = q, A_f = \max(0, -p) = 0.$ gives a purely positive gradient: $\nabla = q \cdot Bp\,\mS = Bpq\,\mS.$
Hence $m_S = pq$ and $m_F = 0$. ReLU is sign-biased: it reinforces successes with the same magnitude as Dr.\ GRPO but completely ignores failures.

\section{Proofs and Additional Theory}
\label{sec:proofs}

Throughout, a prompt $x$ has solve rate $p$ and $q=1-p$. We draw $G$ i.i.d.\ rollouts $\tau_1,\dots,\tau_G\sim\pi_\theta(\cdot\mid x)$ with binary rewards $r_i$, $K=\sum_i r_i$, $\hat p=K/G$, and score $s_i=\nabla_\theta\log\pi_\theta(\tau_i\mid x)=\sum_t\nabla_\theta\log\pi_\theta(a_{i,t}\mid x,a_{i,<t})$.

We make three assumptions.
\begin{itemize}
  \item The loss is a token \emph{sum} with a constant normaliser, as in our training (Appendix~\ref{sec:training_app}).
  \item We analyse the on-policy step, where the PPO ratio equals $1$.
  \item Advantages depend only on the group's rewards.
\end{itemize}
Two facts are used repeatedly:
\[
\mathbb E[s]=0,\qquad \mathbb E[r\,s]=\nabla p,\qquad\text{hence}\qquad p\,\mu_S=-q\,\mu_F=\nabla p,
\]
where $\mu_S=\mathbb E[s\mid r{=}1]$ and $\mu_F=\mathbb E[s\mid r{=}0]$.

\subsection{The optimal baseline equals the solve rate only if successes and failures look alike}\label{app:optimal_baseline}

Write the conditional score energies as
\[
E_S=\mathbb E\big[\|s\|^2\mid r{=}1\big],\qquad E_F=\mathbb E\big[\|s\|^2\mid r{=}0\big],\qquad \nu=E_F/E_S .
\]

\begin{proposition}\label{prop:optimal_baseline}
For the single-sample estimator $g(b)=(r-b)\,s$ with $b$ independent of $\tau$, $\mathbb E[g(b)]=\nabla p$ for every $b$, and $\mathbb E\|g(b)\|^2$ is minimised at
\[
b^*=\frac{\mathbb E[r\|s\|^2]}{\mathbb E[\|s\|^2]}=\frac{pE_S}{pE_S+qE_F},\qquad b^*-p=\frac{pq\,(E_S-E_F)}{pE_S+qE_F}.
\]
Hence $b^*=p$ if and only if $E_S=E_F$. If failures carry more score energy ($E_F>E_S$), then $b^*<p$. Using $b=p$ instead of $b^*$ inflates the second moment by
\[
\frac{\mathbb E\|g(p)\|^2}{\mathbb E\|g(b^*)\|^2}=\frac{(q+p\nu)(p+q\nu)}{\nu}\ \le\ \frac{(1+\nu)^2}{4\nu}.
\]
\end{proposition}

\begin{proof}
Unbiasedness follows from $\mathbb E[s]=0$. Since $r^2=r$,
\[
\mathbb E\|g(b)\|^2=\mathbb E\big[(r-b)^2\|s\|^2\big]=(1-2b)\,\mathbb E[r\|s\|^2]+b^2\,\mathbb E\|s\|^2 ,
\]
a convex quadratic in $b$ minimised at $b^*=\mathbb E[r\|s\|^2]/\mathbb E\|s\|^2$. Conditioning on $r$ gives $\mathbb E[r\|s\|^2]=pE_S$ and $\mathbb E\|s\|^2=pE_S+qE_F$, and the expression for $b^*-p$ follows. Since $\mathrm{tr}\,\mathrm{Cov}\,g(b)=\mathbb E\|g(b)\|^2-\|\nabla p\|^2$, the same $b^*$ minimises the variance.

For the inflation factor, $\mathbb E\|g(b)\|^2=p(1-b)^2E_S+qb^2E_F$. At $b=p$ this is $pq(qE_S+pE_F)$. At $b^*$ it is $pqE_SE_F/(pE_S+qE_F)$. Their ratio is $(qE_S+pE_F)(pE_S+qE_F)/(E_SE_F)=(q+p\nu)(p+q\nu)/\nu$. The two factors sum to $1+\nu$, so their product is at most $(1+\nu)^2/4$.
\end{proof}

AsymGRPO assigns $q$ to successes and $-p/\delta$ to failures, i.e.\ $A=\kappa(r-b)$ with $\kappa=q+p/\delta$ and $b=p/(p+\delta q)$, so the optimal failure scale at a matched expected update, i.e.\ minimising $\mathbb E\|As\|^2/\kappa^2$ is $\delta^*=\nu$. In this sense the sign-balanced choice $\delta=1$ is variance-optimal only when $E_S=E_F$.

\paragraph{Nothing forces $E_S=E_F$.} The first moments are tied by $p\,\mu_S=-q\,\mu_F$, but the second moments are not. Let $\chi_r=\|\mu_r\|^2/\mathbb E[\|s\|^2\mid r]$ be the same-outcome coherence. Then
\[
\nu=\frac{\|\mu_F\|^2/\chi_F}{\|\mu_S\|^2/\chi_S}=\Big(\frac pq\Big)^2\frac{\chi_S}{\chi_F}.
\]
If successes and failures were equally coherent, $E_S=E_F$ would hold only at $p=\tfrac12$.

\paragraph{Length and entropy enter through the token decomposition.} Write $s=\sum_t s_t$ with $s_t=\nabla_\theta\log\pi_\theta(a_t\mid x,a_{<t})$, and $s_t=0$ after the end of the sequence. The token scores are martingale differences, $\mathbb E[s_t\mid a_{<t}]=0$, so $\mathbb E\langle s_t,s_{t'}\rangle=0$ for $t\neq t'$. Conditionally on the outcome, the cross terms need not vanish. Let $u_t=\|s_t\|^2$ and $C_r=\mathbb E\big[\sum_{t\neq t'}\langle s_t,s_{t'}\rangle\mid r\big]$. Then $pC_S+qC_F=0$ and
\[
\nu=\frac{\mathbb E\big[\sum_tu_t\mid r{=}0\big]+C_F}{\mathbb E\big[\sum_tu_t\mid r{=}1\big]+C_S}\ \approx\ \frac{L_F}{L_S}\cdot\frac{\bar u_F}{\bar u_S},
\]
where $L_r=\mathbb E[T\mid r]$ is the mean length and $\bar u_r=\mathbb E[\sum_tu_t\mid r]/L_r$ the mean per-token energy. The approximation drops the cross terms. The score energy therefore grows with length, so failures that are longer, for example because they are truncated at the token budget, raise $E_F$.

The per-token energy is tied to the token distribution. For the output head $W$, with logits $Wh_t$ and next-token distribution $\pi_t$,
\[
u_t^W=\big\|(e_{a_t}-\pi_t)\,h_t^\top\big\|^2=\big(1-2\pi_t(a_t)+\|\pi_t\|^2\big)\|h_t\|^2,\qquad \mathbb E_{a_t\sim\pi_t}\big[u_t^W\big]=\big(1-\|\pi_t\|^2\big)\|h_t\|^2 .
\]
The on-policy mean is the Gini impurity of $\pi_t$, which, like the entropy, is zero on deterministic tokens and grows as $\pi_t$ spreads out. The realised energy is largest when a low-probability token is sampled. Failures that run at higher entropy or contain more low-probability tokens thus have $\bar u_F>\bar u_S$.

\paragraph{When the gap matters.} The inflation factor of Proposition~\ref{prop:optimal_baseline} is flat around $\nu=1$, so small differences between successes and failures cost little:
\begin{center}
\begin{tabular}{lccccc}
\toprule
$\nu$ (or $1/\nu$) & $1.25$ & $1.5$ & $2$ & $3$ & $5$\\
\midrule
worst case over $p$, $(1+\nu)^2/4\nu$ & $1.01$ & $1.04$ & $1.13$ & $1.33$ & $1.80$\\
average over $p\sim U(0.05,0.95)$ & $1.01$ & $1.03$ & $1.09$ & $1.24$ & $1.58$\\
\bottomrule
\end{tabular}
\end{center}
The factor is symmetric in $\nu\mapsto1/\nu$. The second moment at $b=p$ is within $10\%$ of optimal for all $p$ when $\nu\in[0.54,1.86]$, so the gap matters only when one outcome carries about twice the score energy of the other. Length and per-token energy multiply: failures $1.4\times$ longer with $1.35\times$ the per-token energy give $\nu\approx1.9$. Entropy is only a proxy, since extra tail mass raises $H$ more than the Gini impurity. In our runs, failures do sample lower-probability tokens than successes (Figure~\ref{fig:entropy_solved_unsolved}).

\begin{figure}[ht]
    \centering
    \includegraphics[width=0.8\linewidth]{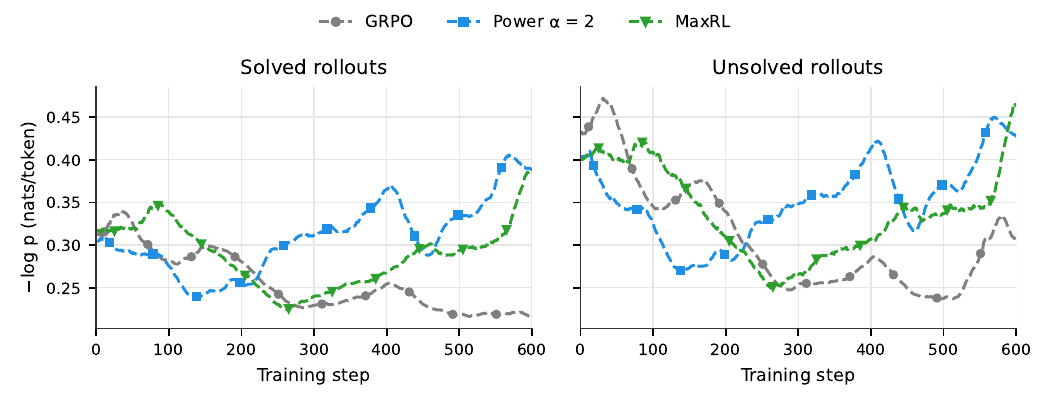}
    \caption{Mean negative log-probability of sampled tokens for solved (left) and unsolved (right) rollouts during training. Unsolved rollouts sample lower-probability tokens for all three advantages, consistent with $\bar u_F>\bar u_S$.}
    \label{fig:entropy_solved_unsolved}
\end{figure}

\section{Proofs for Signed Advantages}
\label{sec:weight_space_details}

This appendix supports Section~\ref{sec:entropy_collapse}. A sign bias gives every group a non-zero net advantage (Appendix~\ref{app:net_adv}). Under Adam, this shift moves every rarely sampled token by the same step, lifting or draining the whole vocabulary tail (Appendix~\ref{sec:proof_rank1}). We then read off the predicted entropy of published advantages (Appendix~\ref{app:entropy_table}) and report the weight-space evidence (Appendix~\ref{app:signed_evidence}). Notation follows Appendix~\ref{sec:proofs}.

\subsection{Sign bias adds a net advantage}\label{app:net_adv}

\paragraph{Net advantage.} With fixed weights, a group with $K$ successes has $\sum_iA_i=Kw_S-(G-K)w_F$. Since $\mathbb E[K]=Gp$,
\begin{equation}
\mathbb E\Big[\textstyle\sum_iA_i\Big]=G\,(p\,w_S-q\,w_F)=G\,(m_S-m_F).
\label{eq:net_adv}
\end{equation}
The per-sample ratio is $|A_+|/|A_-|=w_S/w_F=m_Sq/(m_Fp)$, which equals $q/p$ if and only if $m_S=m_F$ (Figure~\ref{fig:entropy}, right). With group-estimated weights, the FADE/AsymGRPO advantage $A_i=\hat q^{\alpha-1}(r_i-\hat p)$ for successes and $\hat q^{\alpha-1}(r_i-\hat p)/\delta$ for failures gives $G\hat p$ successes with $A_i=\hat q^{\alpha}$ and $G\hat q$ failures with $A_i=-\hat q^{\alpha-1}\hat p/\delta$, so $\sum_iA_i=G\hat p\hat q^{\alpha}(1-1/\delta)$ exactly. Its expectation $G\,\mathbb E[\hat p\hat q^{\alpha}](1-1/\delta)$ has the same sign as $G(m_S-m_F)$ but is smaller at finite $G$, e.g.\ $(G-1)pq(1-1/\delta)$ for AsymGRPO.

\subsection{Tail drift under Adam}\label{sec:proof_rank1}

On-policy, the expected update of any reward-only advantage is $c(p,\delta)\,\nabla_\theta p$ (Equation~\ref{eq:pop_update}), so changing $\delta$ rescales the update but does not rotate it: the entropy change must come from the optimizer. We show this for a linear softmax head.

\begin{proposition}[Net advantage and rank-1 tail drift]\label{prop:tail_rank1}
  Consider a linear softmax head $z = Wh$, $W \in \mathbb{R}^{V \times d}$, and one update on $N$ tokens with hidden states $h_i$, sampled tokens $y_i$ and advantages $a_i$. Row $j$ of the gradient of $\sum_i a_i \log \pi(y_i \mid h_i)$ is $g_j = \sum_i a_i \left(\mathbb{1}[y_i = j] - \pi_{ij}\right) h_i$. Let $\mathcal{U}$ be the set of tokens not sampled in the batch, let $m = \sum_i a_i$ be the net advantage, and assume the tail is unstructured: $\pi_{ij} h_i = \pi_j \bar{h}$ for all $i$ and all $j \in \mathcal{U}$. Then $g_j = -m\, \pi_j\, \bar{h}$ for $j \in \mathcal{U}$, and:
  \begin{enumerate}[label=(\roman*), nosep]
    \item under SGD, $\Delta W_j = -\eta\, m\, \pi_j\, \bar{h}$, which vanishes with $\pi_j$;
    \item under signSGD (Adam with $\beta_1 = \beta_2 = \epsilon = 0$), $\Delta W_{\mathcal{U}} = -\eta \operatorname{sign}(m)\, \mathbf{1}_{\mathcal{U}} \otimes \operatorname{sign}(\bar{h})$: exactly rank 1, identical for every unsampled row and independent of $\pi_j$. At $h = \bar{h}$ every tail logit moves by $-\eta \operatorname{sign}(m) \|\bar{h}\|_1$, so $m < 0$ lifts the whole tail and $m > 0$ drains it;
    \item if $m = 0$, the common term vanishes;
    \item if $\operatorname{sign}(m)$ and $\operatorname{sign}(\bar h)$ are fixed across steps and row $j$ is sampled with probability $\rho_j$ per step, the expected signSGD step along $-\operatorname{sign}(m) \operatorname{sign}(\bar{h})$ is at least $\eta(1 - 2\rho_j)$ per coordinate.
  \end{enumerate}
\end{proposition}

\begin{proof}
  For $j \in \mathcal{U}$ the indicator vanishes, so $g_j = -\sum_i a_i \pi_{ij} h_i = -m\, \pi_j\, \bar{h}$. This gives (i) and (iii), and (ii) follows from $\operatorname{sign}(-m\, \pi_j\, \bar{h}) = -\operatorname{sign}(m) \operatorname{sign}(\bar{h})$ for $\pi_j > 0$, with $\Delta z_j = \Delta W_j^\top \bar{h}$. For (iv), with probability $1 - \rho_j$ the row is unsampled and moves by $\eta$ along the drift direction; otherwise it moves by at most $\eta$ against it.
\end{proof}

Under SGD, the drift on a tail row is proportional to $\pi_j$ and cancels in expectation against the rare steps where $j$ is sampled. Under signSGD, frequent small but consistent signals outweigh rare large ones, so every rarely sampled row drifts by $\mathcal{O}(\eta)$ per step in the same direction, which gives Equation~\ref{eq:adam_tail} and a rank-1 update of the output head. The direction of this drift is set by the sign of the net advantage $m$. With equal trajectory lengths $L$, AsymGRPO has $m = L \sum_{\text{groups}} G\, \hat{p}\hat{q}\,(1 - 1/\delta)$, negative if and only if $\delta < 1$; Dr.~GRPO and Power $\alpha$ are sign-balanced ($m = 0$); and REINFORCE with $\pm 1$ rewards has $m = L \sum_{\text{groups}} G\,(\hat p - \hat q)$, negative whenever most prompts have $\hat p<\frac12$, as in our training (average solve rate $0.3$ at 7B). This matches the sign-biased (A) versus sign-balanced (S) split in Table~\ref{tab:rank1_main}.\footnote{With a token-sum loss and failures longer than successes, Dr.~GRPO also has a small negative $m$.}

The result is exact only for signSGD, the $\beta_1 = \beta_2 = \epsilon = 0$ limit of Adam. With $\beta_2 = 0.999$, the same rectification appears in simulation for rows sampled much less often than once every $1/(1 - \beta_2)$ steps, but we do not prove it. The unstructured-tail assumption also ignores correlations between tail probabilities and outcomes, so the proposition explains the rank-1, uniform-sign structure of the head update and the direction of the entropy change, but not its magnitude.

\subsection{Predicted entropy of published advantages}\label{app:entropy_table}

Since the net advantage is $G(m_S-m_F)$ (Equation~\ref{eq:net_adv}), the direction of the entropy drift of any advantage can be read off its masses in Table~\ref{tab:weights} (Table~\ref{tab:entropy_stability}). A net-positive advantage drains the tail until prompts reach $p\in\{0,1\}$, where a group carries no signal since all its rewards agree, so training starves. A net-negative advantage lifts the tail and drifts.

\begin{table}[h]
\centering
\caption{\textbf{Predicted entropy and stability of published advantages}, grouped by the sign of the net advantage $G(m_S-m_F)$ (Equation~\ref{eq:net_adv}; masses from Table~\ref{tab:weights}). Under Adam, a negative net advantage lifts the vocabulary tail and a positive one drains it (Proposition~\ref{prop:tail_rank1}).}
\label{tab:entropy_stability}
\footnotesize
\setlength{\tabcolsep}{4pt}
\renewcommand{\arraystretch}{1.2}
\begin{tabular}{@{}>{\raggedright\arraybackslash}p{0.55\linewidth} >{\raggedright\arraybackslash}p{0.41\linewidth}@{}}
\toprule
\textbf{Method} & \textbf{Entropy $\to$ stability}\\
\midrule
\rowcolor[gray]{0.92}\multicolumn{2}{@{}l}{\emph{Sign-balanced} ($m_S=m_F$)}\\
Dr.~GRPO, GRPO, RLOO, Power $\alpha$, MaxRL, LME, analytic pass@$k$ & slow decay $\to$ stable\\
\rowcolor[gray]{0.92}\multicolumn{2}{@{}l}{\emph{Net-negative: tail lift}}\\
AsymGRPO $\delta<1$ & rises $\to$ drift, early gains fade\\
Asym.\ power $(2,\frac12)$ & rises $\to$ drift\\
AsymRL~\citep{arnal2026asymmetric} & slight rise $\to$ mild drift\\
W-REINFORCE / NSR~\citep{w_reinforce}, $r\in\{\pm1\}$ & rises if $p<\frac{1}{1+\lambda}$ $\to$ drift over long runs\\
CoRPO~\citep{CorPO} & rises on hard prompts $\to$ mild drift\\
\rowcolor[gray]{0.92}\multicolumn{2}{@{}l}{\emph{Net-positive: tail drain}}\\
AsymGRPO $\delta>1$ & collapses $\to$ starvation\\
REINFORCE ($r\in\{0,1\}$), ReLU~\citep{srinivasan2018actor} & collapses $\to$ starvation\\
\rowcolor[gray]{0.92}\multicolumn{2}{@{}l}{\emph{Sign depends on the solve rate}}\\
MC-GRPO~\citep{kim2026mc}, Asym.\ power $(2,2)$ & self-correcting $\to$ stable\\
Constant baseline $r-C$ & lifts hard, drains easy $\to$ amplifies collapse\\
FADE ($\delta<1$ only when $H<H^*$) & held near $H^*$ $\to$ stable\\
\bottomrule
\end{tabular}
\end{table}

\subsection{Empirical evidence}\label{app:signed_evidence}

\paragraph{Rank-1 lock-in.} We measure how rank-1 the output-head update $W_{\Delta} = W_\mathrm{rl} - W_\mathrm{sft}$ is by $r_1 = \sigma_1^2(W_\Delta) / \|W_\Delta\|_F^2$, the share of its Frobenius energy in the leading singular component. All methods start rank-1 dominant. Sign-balanced methods gradually escape (Figure~\ref{fig:svd_evolution}), while failure-biased methods such as AsymGRPO with $\delta < 1$ and REINFORCE with $\pm1$ rewards remain locked in rank-1, with RL change concentrating almost entirely in the output head (up to ${\sim}97\%$ of $\|W_\Delta\|$ at 7B, Table~\ref{tab:rank1_main}). This lock-in comes with lower diversity: on LiveCodeBench v6, AsymGRPO ($\delta=0.5$) reaches pass@$100$ of $60.5$ against $62.5$ for Dr.~GRPO and $66.0$ for Power $\alpha{=}2$ (Table~\ref{tab:rl-results-454}).

\begin{figure}[ht!]
    \centering
    \includegraphics[width=0.6\columnwidth]{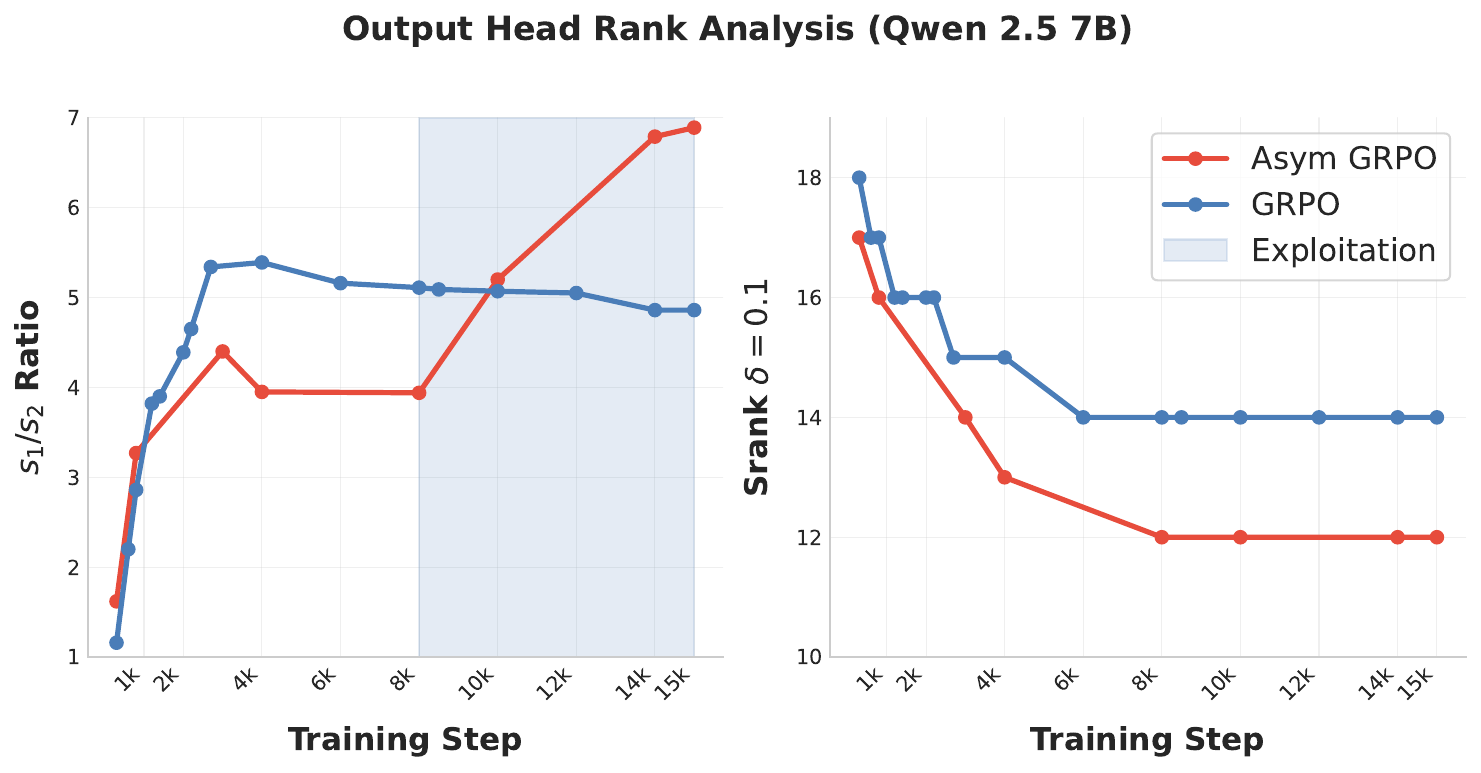}
    \caption{Failure-biased methods converge to a lower-rank weight update than sign-balanced methods as shown by the singular-value ratio $s_1/s_2$ of the output head of $W_\Delta = W_\mathrm{rl} - W_\mathrm{sft}$, and its $\mathrm{srank}_{\delta} = \min \{ k : \sum_{i=1}^{k} s_i \,/\, \sum_{i=1}^{d} s_i \geq 1 - \delta \}$ \citep{kumarimplicit}.}
    \label{fig:svd_evolution}
\end{figure}

\begin{table}[t]
\centering
\footnotesize
\caption{Rank-1 dominance of the output head $\Delta W$. \textbf{(a)}~At 500 steps, all methods are rank-1 dominant. \textbf{(b)}~Over longer training, sign-biased methods maintain rank-1 while sign-balanced methods lose it; the effect weakens with model scale.}
\label{tab:rank1_main}
\begin{minipage}[t]{0.28\textwidth}
\centering
\textbf{(a) 500 steps (7B)}\\[4pt]
\begin{tabular}{@{}lc@{}}
\toprule
Method & Rank-1\% \\
\midrule
REINFORCE & 95.4\% \\
Dr.~GRPO & 91.4\% \\
AsymGRPO & 92.4\% \\
Power $\alpha{=}2$ & 83.9\% \\
\bottomrule
\end{tabular}
\end{minipage}%
\hfill
\begin{minipage}[t]{0.70\textwidth}
\centering
\textbf{(b) Converged}\\[4pt]
\begin{tabular}{@{}llcccc@{}}
\toprule
Scale & Method & Sym & OH frac & $s_1/s_2$ & Rank-1\% \\
\midrule
7B 8k  & AsymGRPO($\delta{=}0.3$) & A & $\sim$97\% & 7.07 & $\sim$78\% \\
7B 16k & AsymGRPO($\delta{=}0.5$) & A & 88.5\% & 6.96 & 96\% \\
7B 8k  & REINFORCE                & A & ---        & 8.51 & $\sim$81\% \\
\midrule
7B 8k  & Dr.~GRPO               & S & $<$10\%   & 3.98 & 62\% \\
7B 8k  & Power $\alpha{=}2$       & S & $<$10\%   & 3.61 & 61\% \\
14B    & Dr.~GRPO               & S & 3.2\%     & 4.52 & 11\% \\
\midrule
32B    & AsymGRPO($\delta{=}0.8$) & A & 3.96\%    & 8.87 & 45.6\% \\
32B    & Dr.~GRPO               & S & 0.69\%    & 2.24 & 2.2\% \\
\bottomrule
\end{tabular}
\end{minipage}
\end{table}

\paragraph{A uniform tail direction.} The rank-1 left singular vector $u \in \mathbb{R}^{|\mathcal{V}|}$ of the output head delta reveals a universal uniform-sign direction shared across sign-biased runs: 99.93--99.97\% of tokens share the same sign in $u$ (Chinese, Russian, Thai, Arabic text; non-code identifiers), as predicted by Proposition~\ref{prop:tail_rank1}. Since the sign of $u$ is only defined jointly with the right singular vector, whether these logits rise or fall must be read from $\log \pi_\mathrm{rl} - \log \pi_\mathrm{sft}$; Proposition~\ref{prop:tail_rank1} predicts a rise for $\delta < 1$. Moreover, $u$ is not sparse (50\% of energy in ${\sim}$22\% of vocab), and $u$ vectors are nearly identical across sign-biased runs ($\cos = 0.93$--$0.99$) and context lengths ($\cos = 0.82$--$0.96$).

\paragraph{The drift accumulates.} If the output-head gradient at every step is close to $w_t \otimes u$ for a shared direction $u$, the accumulated update $\sum_t \eta_t g_t \approx (\sum_t \eta_t w_t) \otimes u$ is rank-1. The cosine similarity between successive batch-gradient directions at the output head reaches $|\cos| \approx 1.0$ from step~${\sim}600$ onward for sign-biased methods, and stays lower for sign-balanced ones.

\paragraph{Successes are more correlated than failures.} Correct trajectories are $2.4$--$4.4{\times}$ more correlated than failed ones in the hidden-state subspace orthogonal to the leading update direction, across methods and training steps (Table~\ref{tab:rho}; Figure~\ref{fig:entropy_ratios}, left).

\begin{table}[h]
\centering
\caption{\textbf{Successes are more correlated than failures.} Mean pairwise cosine similarity $\hat\rho$ ($\times10^{-3}$) of last-layer hidden states, after projecting out the leading direction, for correct and failed trajectories.}
\label{tab:rho}
\begin{tabular}{llccc}
\toprule
Method & Step & Correct & Failed & Ratio\\
\midrule
AsymGRPO & 300 & 17.7 & 4.6 & 3.8$\times$\\
AsymGRPO & 3k  & 19.3 & 5.7 & 3.4$\times$\\
AsymGRPO & 10k & 15.7 & 5.4 & 2.9$\times$\\
AsymGRPO & 15k & 12.6 & 3.2 & 3.9$\times$\\
Dr.~GRPO & 8k  & 12.1 & 5.0 & 2.4$\times$\\
Dr.~GRPO & 60k & 29.4 & 6.7 & 4.4$\times$\\
\bottomrule
\end{tabular}
\end{table}

\paragraph{Analysis setup.} We analyze the weights of three model scales:
\begin{itemize}[nosep]
    \item Qwen~2.5~7B: $d{=}3584$, 28~layers, 28~heads, 4~KV heads, vocab${=}152{,}064$. 26 runs at 8k context, 8 runs at 16k.
    \item Qwen~2.5~14B: $d{=}5120$, 48~layers, 40~heads, 8~KV heads, vocab${=}152{,}064$. 1 run (Dr.~GRPO), 5 checkpoints.
    \item CWM~32B: $d{=}6144$, 64~layers, 48~heads, vocab${=}128{,}256$. 6 runs.
\end{itemize}
For each run we compute $W_\Delta = W_\mathrm{rl} - W_\mathrm{sft}$ and measure global $L_2$ distance, per-layer $L_2$ decomposition, pairwise cosine similarity of weight-change vectors with hierarchical clustering, full SVD of the output head delta, rank-1 decomposition $W_\Delta \approx s_1\, u\, v^\top$ with cross-run alignment of $u/v$ vectors, and per-layer SVD across all weight matrices. For the hidden-state correlations, we extract the last-layer hidden states $h_i \in \mathbb{R}^d$ for all tokens across a batch, project out the top right singular vector $u_1$ of the hidden-state matrix, $h_i^\perp = h_i - (h_i^\top u_1) u_1$, and compute the mean pairwise cosine similarity $\hat\rho$ of these $h^\perp$ vectors separately for correct and failed trajectories.

\section{Pass@$k$ Scaling Law: Proofs, Fits and Simulations}
\label{sec:proof_theorem}

We frame diversity as: how fast does pass@$k$ increase with $k$? If we define $\text{pass@}k = \mathbb{E}_p[1 - (1-p)^k] = 1 - \mathbb{E}_p[(1-p)^k]$ where $p$ is the per-problem pass@$1$ rate, the asymptotic behavior depends on the distribution of $p$ near $0$.
\subsection{Theory of Scaling Laws of Pass@k}
\label{sec:proof_scaling_laws}

We adapt the proof from \citet{schaeffer25a}. Let $p_D(p)$ be the density of pass@$1$ rates over a dataset $D$, and $\mathrm{pass}_{D@k} = \mathbb{E}_{p \sim D}[1 - (1 - p)^k]$.

\begin{theorem}[Two scaling regimes]
The asymptotic behavior of $\mathrm{pass}_{D@k}$ as $k \to \infty$ is determined by $p_D(p)$ near $p = 0$:
\begin{enumerate}[nosep]
\item \textbf{Power-law regime:} If $p_D(p) = C p^{b-1} + O(p^{b-1+\theta})$ as $p \to 0^+$ ($C,b,\theta > 0$), then $\mathrm{pass}_{D@k} = 1 - C\,\Gamma(b)\,k^{-b} + o(k^{-b})$.
\item \textbf{Rapidly decaying regime:} If $p_D(p) \le C \exp(-c / p^\alpha)$ as $p \to 0^+$ ($c,\alpha > 0$), then $1 - \mathrm{pass}_{D@k} = o(k^{-b})$ for all $b > 0$.
\end{enumerate}
\end{theorem}

\begin{proof}
Since $(1-p)^k \approx e^{-kp}$ for large $k$, we have $1 - \mathrm{pass}_{D@k} \approx \int_0^1 e^{-kp}\,p_D(p)\,dp$, a Laplace-type integral dominated by small~$p$.

\textbf{Case 1.} With $p_D(p) \sim Cp^{b-1}$, the tail $[\epsilon,1]$ is exponentially small in $k$, so
$1 - \mathrm{pass}_{D@k} \approx C\int_0^\infty p^{b-1} e^{-kp}\,dp = C\,k^{-b}\,\Gamma(b)$,
where the last step uses $u = kp$.

\textbf{Case 2.} The integrand $e^{-kp - c/p^\alpha}$ is maximized at $p^* \sim k^{-1/(1+\alpha)}$, giving $\int_0^1 e^{-kp}\,p_D(p)\,dp \le \exp(-\Omega(k^{\alpha/(1+\alpha)}))$, which decays faster than any power of $k$.
\end{proof}

\subsection{How the advantage moves the power law}
\label{app:rl_scaling}

\paragraph{Speed of a problem.} We give each problem its own log-odds $\ell=\ln\frac{p}{1-p}$ and treat it as a free parameter. On-policy, the expected update is $c(p)\,\nabla p$ (Equation~\ref{eq:pop_update}), and $\partial p/\partial\ell=p(1-p)$, so a gradient step of size $\eta$ moves the log-odds by $\eta\,u_c(p)$ with $u_c(p)=c(p)\,p(1-p)$. This ignores transfer between problems, which shared parameters introduce.

\paragraph{Stretch changes the slope.} Near $p=0$, $\ell\approx\ln p$. If the fraction of problems with solve rate below $p$ is $\propto p^{b}$, then $\Pr(\ell<t)\propto e^{bt}$. After $\ell'=\Lambda+\nu\ell$, $\Pr(\ell'<t)=\Pr\big(\ell<(t-\Lambda)/\nu\big)\propto e^{bt/\nu}$, so the asymptotic slope becomes $b_{\text{eff}}=b/\nu$; over a finite range of $k$, the fitted slope moves in the same direction. Independent light-tailed noise $\varepsilon$ added to $\ell'$ multiplies the tail by $\mathbb E[e^{-b\varepsilon/\nu}]$, so it leaves $b_{\text{eff}}$ unchanged and only shifts the level $a$.

\paragraph{Which way the slope moves.} After one small step, $\ell'=\ell+\eta\,u_c(\ell)$. Regressing $\ell'$ on $\ell$ over the current problems gives slope $\nu=1+\eta\,\mathrm{Cov}(\ell,u_c)/\mathrm{Var}(\ell)$, and the mean log-odds gain is $dm=\eta\,\mathbb E[u_c]$. Hence
\begin{equation}
\frac{d\ln b_{\text{eff}}}{dm}=-\kappa_c,\qquad
\kappa_c=\frac{\mathrm{Cov}(\ell,u_c)}{\mathrm{Var}(\ell)\,\mathbb E[u_c]},
\label{eq:kappa}
\end{equation}
so the curve flattens when easier prompts move faster ($\kappa_c>0$), keeps its slope when all prompts move together ($\kappa_c=0$), and steepens otherwise. Integrating over steps, with the moments re-evaluated on the current problems, gives the zero-parameter predictions of Figure~\ref{fig:affine_law}.

\begin{figure}[h]
    \centering
    \includegraphics[width=\linewidth]{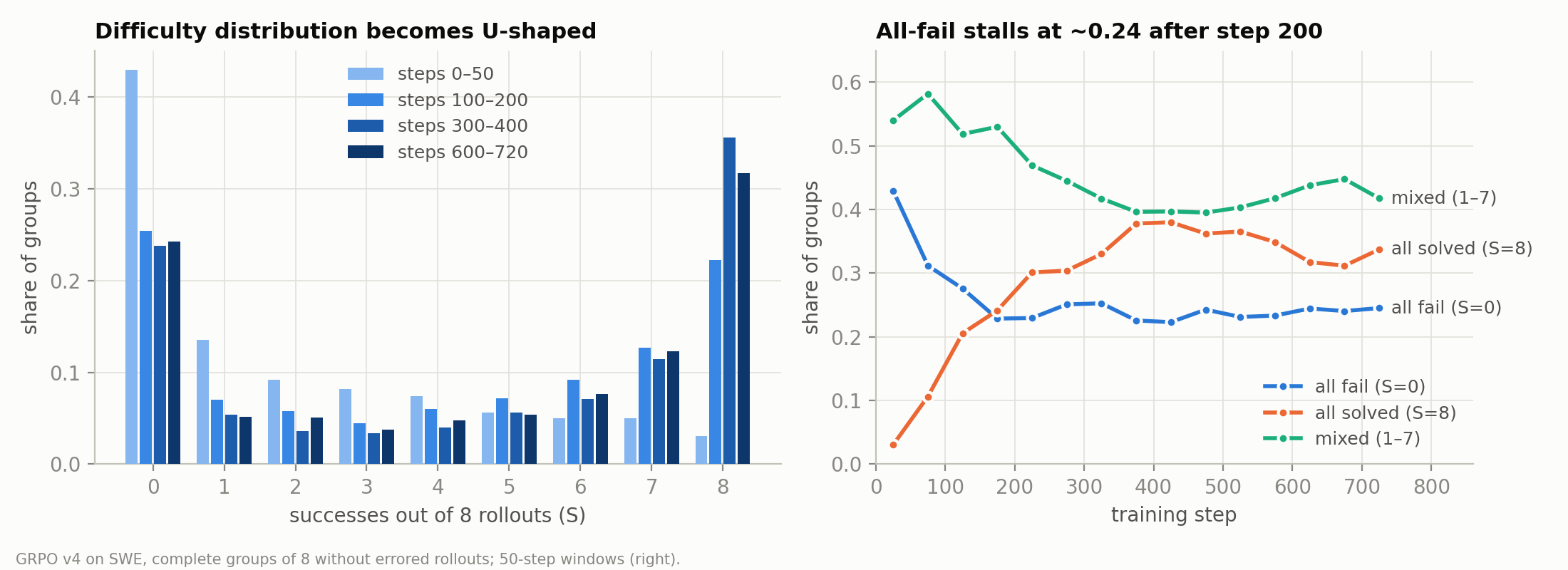}
    \caption{\textbf{Outcomes polarize, and the all-fail share stalls once the shift saturates.} Dr.~GRPO on SWE, complete groups of $8$ rollouts. \textbf{(Left)} Share of prompt groups by number of successes, in four training windows: outcomes concentrate at $0$ and $8$. \textbf{(Right)} Share of all-fail, all-solved and mixed groups over training ($50$-step windows): the all-fail share stalls at $\approx0.24$ after step $200$, when the shift $\Lambda$ stops growing.}
    \label{fig:ushape}
\end{figure}

\subsection{From Asymptotic Laws to Practical Predictions}
\label{app:passk_fits}
\label{app:passk_recipe}

The asymptotic regime is set by the mass of solve rates near $p=0$. For fitting pass@$k$ over a finite range of $k$, both the pure power law and the exponential law overestimate pass@$k$ at small $k$, so we use a shifted power law $G(k)=\exp\bigl(-a(k+k_0)^{-b}\bigr)$ with scale $a$, exponent $b$ and horizontal shift $k_0$. The fitted $b$ is an effective slope over the measured range of $k$: it matches the tail exponent of Appendix~\ref{sec:proof_scaling_laws} only asymptotically, and it is what we report as the slope in the main text. It has the lowest mean squared error across problem-difficulty splits (Figure~\ref{fig:difficulty_fits}) and tracks the pass@$k$ curve throughout RL training (Figure~\ref{fig:scaling_laws_rl}).

\paragraph{Predicting pass@$k$ from the advantage.} The same pieces give a rough forecast from the base model alone. We estimate each prompt's solve rate from base samples and treat every prompt as movable, since we do not know in advance which prompts will never be solved. Each prompt's log-odds then moves at speed $u_c=c(p)\,p(1-p)$, and Equation~\ref{eq:kappa}, re-evaluated as solve rates rise, gives the stretch $\nu$ as a function of the mean log-odds gain. The predicted curve has effective slope roughly $b/\nu$, while the shift, which depends on the learning rate and the length of training, sets its level. Because prompts that are never solved are treated as movable, the forecast is optimistic at large $k$ for long runs.

\begin{figure}[h]
    \centering
    \includegraphics[width=\linewidth]{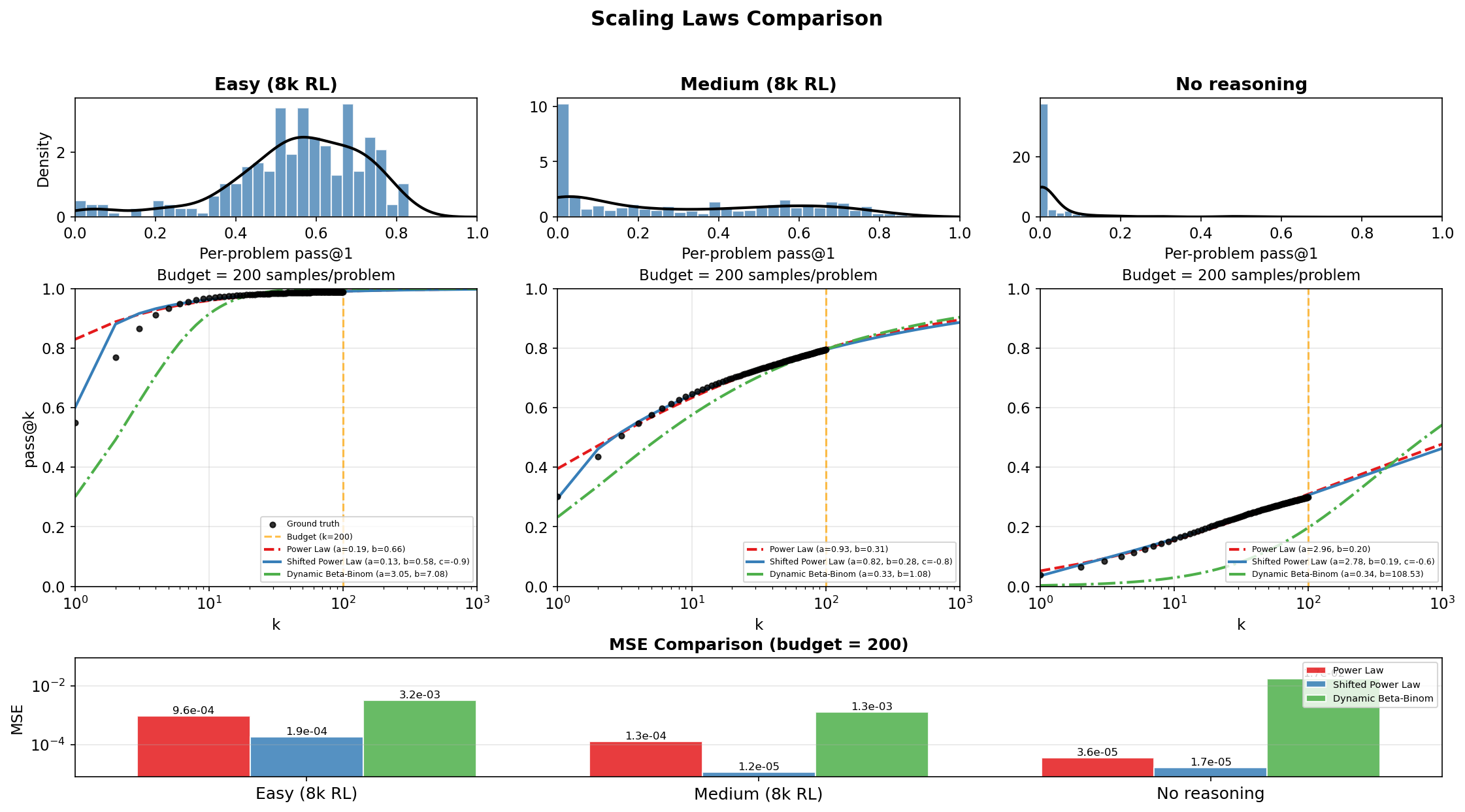}
    \caption{\textbf{Shifted Power Law Predicts Pass@k} across different difficulty splits for an estimated pass@$1$ distribution (from 200 samples) with small (left), medium (middle) and high mass near 0.}
    \label{fig:difficulty_fits}
\end{figure}

\begin{figure}[h]
    \centering
    \includegraphics[width=\linewidth]{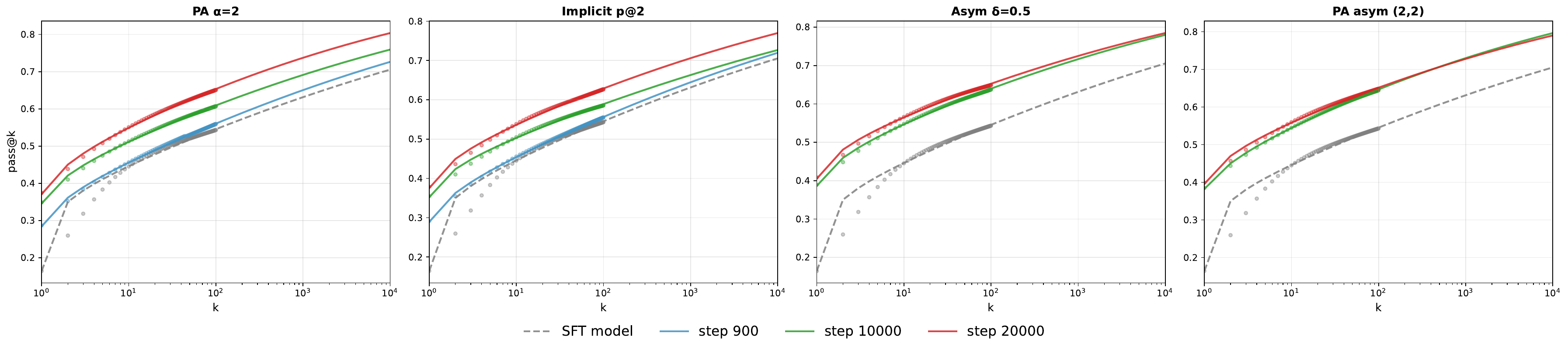}
    \caption{\textbf{Evolution of the k vs.\ pass@k curve during RL} for different advantage functions with similar pass@100 performances.}
    \label{fig:scaling_laws_rl}
\end{figure}

\subsection{Simulations}
\label{app:simulations}

The simulations below illustrate the theory of Appendix~\ref{app:rl_scaling} on a synthetic base. They are not fit to our training runs, so we rely on the directions they show, not on their magnitudes.

\paragraph{Shared model.} Each prompt is a single solve rate $p$, tracked through its log-odds $\ell=\ln\frac{p}{1-p}$, and pass@$k=1-\mathbb E[(1-p)^k]$ over prompts. Base solve rates follow a Beta distribution, whose small-$p$ exponent sets the tail $f(p)\propto p^{b-1}$ and hence the asymptotic power-law slope $b$. Under RL, a prompt's log-odds moves at $u_c(p)=c(p)\,p(1-p)$ per unit step, with $c\equiv1$ for Dr.~GRPO and $c=K(1-p)^{K-1}$ for pass@$K$. SFT is modelled as $u=1-p$, a log-likelihood push that does not vanish as $p\to0$. Prompts do not transfer to each other, except through the sharpening term below.

\paragraph{Shift, stretch and noise (Figure~\ref{fig:intuition}).}\label{app:intuition_sim}
The base is Beta$(0.4,0.85)$ (pass@$1=.32$, $b=0.4$), chosen to resemble the SWE base. Dr.~GRPO updates each prompt's log-odds as
\[
\ell\;\leftarrow\;\ell+\eta_t\big[p\,q+\lambda(\ell-\ell^*)\big]+\sigma\,\eta_t\sqrt{p}\;\varepsilon,\qquad \varepsilon\sim\mathcal N(0,1),
\]
and SFT as $\ell\leftarrow\ell+\eta\,(1-p)+\sigma_0\,\varepsilon$.
\begin{itemize}[nosep,leftmargin=*]
  \item \emph{Shrinking steps.} $\eta_t$ decays over training, standing in for the shift slowing down: in our runs the shift per step falls about $10\times$ by steps $150$--$200$ at a constant learning rate, and the all-fail share stops falling around the same time (Figure~\ref{fig:ushape}).
  \item \emph{Sharpening.} The term $\lambda(\ell-\ell^*)$ stands in for entropy collapse through shared parameters: prompts whose correct answer is already likely get more confident, and prompts that rely on rare correct paths lose them. Without it no prompt can get worse, and RL could only rejoin the base curve at large $k$, never fall below it. We set $\lambda$ so that the Dr.~GRPO stretch ($\nu\approx1.35$) is close to the $1.2$--$1.3$ we fit on SWE.
  \item \emph{Noise.} Its scale is proportional to $\eta_t$, so it fades with the step size, and to $\sqrt p$, so it is negligible in the tail and does not affect the slope. This is a simplification: in our runs prompts do not all gain the same amount from a shared update, and this spread around the average gain is the noise. It is therefore proportional to the shift and does not depend on difficulty.
\end{itemize}
In this model SFT stretches the log-odds by $\nu=0.73$ and Dr.~GRPO by $\nu=1.18$ early and $1.09$ late, and Dr.~GRPO falls below the base curve at large $k$. Where this loss starts depends on $\lambda$ and $\ell^*$; that it happens does not.

\section{FADE Scheduler Additional Results}
\label{sec:appendix_fade}

Algorithm~\ref{alg:fade} details the full FADE procedure. Two exponential moving averages track the online solve rate $\hat{p}$ and policy entropy $\hat{H}$, which respectively control the difficulty focus $\alpha$ and the sign bias $\delta$. Crucially, $\alpha$ and $\delta$ are updated from the smoothed EMAs, while the per-sample advantage uses the group-level mean $\bar{r}$ over the $G$ rollouts of the same prompt. For binary rewards, the resulting per-prompt gradient masses are $m_S = pq^\alpha$ and $m_F = pq^\alpha/\delta$, so the sign ratio is $m_S/m_F = \delta$: when entropy exceeds $H^*$, $\delta = 1$ (sign-balanced); as entropy drops, $\delta$ decreases and the negative mass dominates, driving entropy recovery.

\paragraph{Hyperparameters and dynamics.} In our experiments we clip $\delta$ to $[0.3,1]$ (Algorithm~\ref{alg:fade}) and set the target entropy $H^*$ to half the initial entropy. In a typical run (Figure~\ref{fig:fade-dynamics}), $\alpha$ starts high and decays as solve rates rise, while $\delta$ stays at $1$ until entropy reaches $H^*$ and then decays. The initial sign-balanced phase prevents premature rank-1 collapse of the update, and the decaying $\alpha$ avoids overfitting to hard problems. FADE reaches peak pass@$1$ earlier than the best static baseline (Power $\alpha{=}2$): 20k steps earlier at the 7B scale (30k vs.\ 50k) and 2.5k steps earlier at the 32B scale (15k vs.\ 17.5k; Table~\ref{tab:rl-results-v5} and Figure~\ref{fig:pass_at_k_scheduler}).

\paragraph{Entropy target controls the update rank.} FADE's entropy target is a single interpretable knob for exploration versus exploitation in the learned update $W_\Delta$. We track the SVD of the output-head weight change $W_\Delta$ (singular-value ratio $s_1/s_2$ and rank-1 fraction) alongside its $L_2$ norm across three FADE runs with $H^* \in \{0.5, 1.0, 1.3\}$ on Qwen 2.5 7B (Table~\ref{tab:fade_oh}). The result is a clean, continuous transition: $H^*{=}0.5$ keeps $\|W_\Delta\|$ small and the update full-rank; $H^*{=}1.0$ grows $\|W_\Delta\|$ to 67\% with moderate rank concentration; $H^*{=}1.3$ reproduces the rank-1 collapse seen in static failure-biased methods ($s_1/s_2 > 5$, rank-1 fraction 91\%).

\begin{figure*}[t]
    \centering
    \begin{minipage}[c]{0.55\linewidth}
        \centering
        \includegraphics[width=\linewidth]{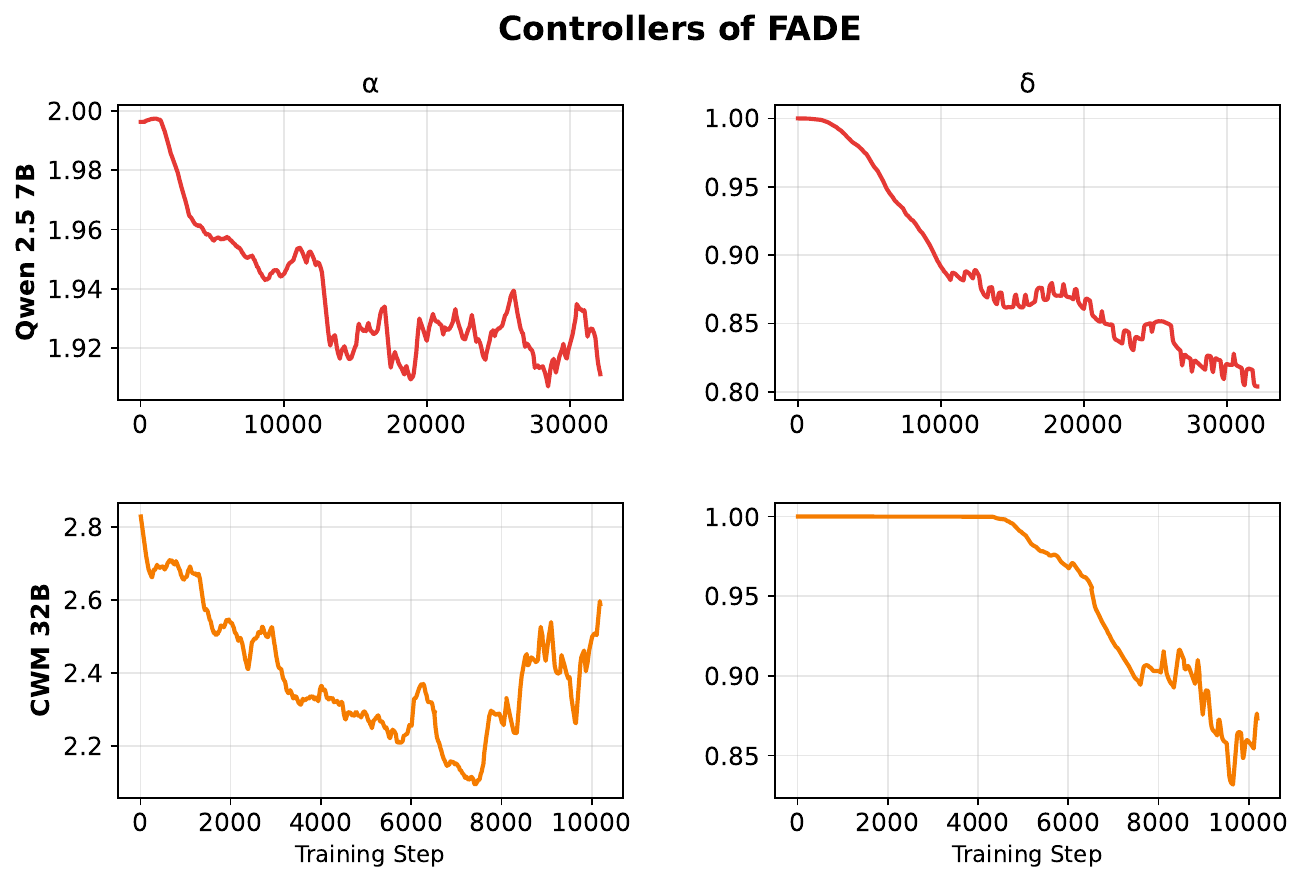}
    \end{minipage}%
    \hfill
    \begin{minipage}[c]{0.42\linewidth}
        \centering
        \footnotesize
        \setlength{\tabcolsep}{2pt}
        \renewcommand{\arraystretch}{1.4}
        \begin{tabular}{@{}llcccc@{}}
        \toprule
        $H^*$ & Step & OH $L_2$ & $s_1/s_2$ & rank-1\% & Inner \\
        \midrule
        0.5 & 10k  & 6.1\%  & 4.10 & 38.7\% & 3.32\% \\
        0.5 & 33k  & 6.3\%  & 3.84 & 54.4\% & 3.32\% \\
        \midrule
        1.0 & 10k  & 13.2\% & 3.51 & 55.9\% & 3.07\% \\
        1.0 & 40k  & 66.9\% & 5.34 & 88.2\% & 1.17\% \\
        \midrule
        1.3 & 10k  & 23.9\% & 3.25 & 68.3\% & 2.69\% \\
        1.3 & 30k  & 72.7\% & 5.77 & 91.2\% & 0.96\% \\
        \midrule
        \multicolumn{2}{@{}l}{\emph{Asym.}} & ${\sim}$90\% & 6.96--8.51 & 78--96\% & ${\sim}$0.3\% \\
        \multicolumn{2}{@{}l}{\emph{Sym.}}  & $<$10\% & 3.61--3.98 & 61--62\% & ${\sim}$3.2\% \\
        \bottomrule
        \end{tabular}
    \end{minipage}
    \caption{\textbf{(Left)} FADE controller dynamics: $\alpha$ and $\delta$ adapt the policy gradient weight $p(1{-}p)^\alpha/\delta$; $\delta$ decays once entropy hits $H^*$ and $\alpha$ decays as solve rates rise. \textbf{(Right)} Output-head SVD metrics at different entropy targets confirm that $H^*$ continuously controls the sign-balanced-to-failure-biased transition (Table~\ref{tab:fade_oh}).}
    \label{fig:fade-dynamics}
    \label{tab:fade_oh}
\end{figure*}

\begin{algorithm}[t]
\caption{Power $\alpha$ (static)}
\label{alg:power_alpha}
\begin{algorithmic}[1]
\REQUIRE Policy $\pi_\theta$, prompts $\mathcal{Q}$, unit tests $\mathcal{T}_q$, power $\alpha \geq 1$, group size $G$, clip range $[1{-}\varepsilon_{\text{low}}, 1{+}\varepsilon_{\text{high}}]$ (Table~\ref{tab:training-setup})
\FOR{step $t = 1, \ldots, T$}
    \STATE Sample $G$ rollouts per prompt; score $r_i \leftarrow \textsc{Exec}(\tau_i, \mathcal{T}_q)$
    \STATE For each prompt, $\bar{r} \leftarrow \frac{1}{G}\sum_{i=1}^{G} r_i$ over its $G$ rollouts
    \FOR{each rollout $i \in \mathcal{B}$}
        \STATE $A_i \leftarrow C_\alpha\,(1{-}\bar{r})^{\alpha-1}(r_i - \bar{r})$, \quad $C_\alpha=\frac{(1+\alpha)^{1+\alpha}}{4\alpha^\alpha}$
    \ENDFOR
    \STATE Update $\theta$ via PPO clipping to $[1{-}\varepsilon_{\text{low}}, 1{+}\varepsilon_{\text{high}}]$ using advantages $\{A_i\}$
\ENDFOR
\end{algorithmic}
\end{algorithm}

\definecolor{fadehl}{RGB}{255,235,180}

\begin{algorithm}[t]
\caption{FADE: Focal Advantage with Dynamic Entropy \colorbox{fadehl}{(additions over Alg.~\ref{alg:power_alpha})}}
\label{alg:fade}
\begin{algorithmic}[1]
\REQUIRE Policy $\pi_\theta$, prompts $\mathcal{Q}$, unit tests $\mathcal{T}_q$, group size $G$, clip range $[1{-}\varepsilon_{\text{low}}, 1{+}\varepsilon_{\text{high}}]$, \colorbox{fadehl}{target entropy $H^*$, EMA coefficient $\beta{=}0.02$}
\STATE \colorbox{fadehl}{Initialize $\hat{p} \leftarrow 0.5$, \; $\hat{H} \leftarrow H_0(\pi_\theta)$}
\FOR{step $t = 1, \ldots, T$}
    \STATE Sample $G$ rollouts per prompt; score $r_i \leftarrow \textsc{Exec}(\tau_i, \mathcal{T}_q)$
    \STATE For each prompt, $\bar{r} \leftarrow \frac{1}{G}\sum_{i=1}^{G} r_i$; \quad \colorbox{fadehl}{$H_t \leftarrow \frac{1}{|\mathcal{B}|}\sum_{i} \frac{1}{T_i}\sum_{t} {-}\log \pi_\theta(a_t^{(i)} \mid q, a_{<t}^{(i)})$}
    \STATE \colorbox{fadehl}{$\hat{p} \leftarrow (1{-}\beta)\, \hat{p} + \beta\, \bar{r}$, \quad $\hat{H} \leftarrow (1{-}\beta)\, \hat{H} + \beta\, H_t$}    \STATE \colorbox{fadehl}{$\alpha \leftarrow \tfrac{1 - \hat{p}}{\hat{p}}$, \quad $\delta \leftarrow \operatorname{clip}\!\bigl(1 + \hat{H} - H^*,\; 0.3,\; 1\bigr)$}
    \FOR{each rollout $i \in \mathcal{B}$}
        \STATE $A_i \leftarrow (1{-}\bar{r})^{\alpha-1}(r_i - \bar{r}) \;/\; \colorbox{fadehl}{$\begin{cases} 1 & r_i \geq \bar{r} \\ \delta & r_i < \bar{r} \end{cases}$}$
    \ENDFOR
    \STATE Update $\theta$ via PPO clipping to $[1{-}\varepsilon_{\text{low}}, 1{+}\varepsilon_{\text{high}}]$ using advantages $\{A_i\}$
\ENDFOR
\end{algorithmic}
\end{algorithm}

\paragraph{Ablations.} We ablate each component independently (Figure~\ref{fig:ablations_fade}): removing $\delta$ causes entropy collapse while removing $\alpha$ loses diversity. We also test two simplifications: (1)~replacing both signals with a deterministic logarithmic schedule, and (2)~driving both $\alpha$ and $\delta$ from entropy alone. Both degrade diversity (pass@10, pass@100), confirming that $\alpha$ and $\delta$ must respond to distinct signals to balance exploration and exploitation. For the deterministic-schedule ablation, we replace both adaptive signals with logarithmic decays,
\begin{align}
    \delta(t) &= \delta_0 - (\delta_0 - \delta_{\min})\,\frac{\log(1 + t/\tau_\delta)}{\log(1 + T/\tau_\delta)}, \nonumber \\
    &\quad \delta_0 = 1.0,\; \delta_{\min} = 0.802,\; \tau_\delta = 12264, \label{eq:log_delta} \\
    \alpha(t) &= \alpha_0 - (\alpha_0 - \alpha_{\min})\,\frac{\log(1 + t/\tau_\alpha)}{\log(1 + T/\tau_\alpha)}, \nonumber \\
    &\quad \alpha_0 = 2.0,\; \alpha_{\min} = 1.917,\; \tau_\alpha = 8500, \label{eq:log_alpha}
\end{align}
where $T$ is the total number of training steps. The constants are fitted to the $\alpha$ and $\delta$ trajectories of the adaptive FADE run on Qwen 2.5 7B.

\begin{figure}[t]
    \centering
    \includegraphics[width=\linewidth]{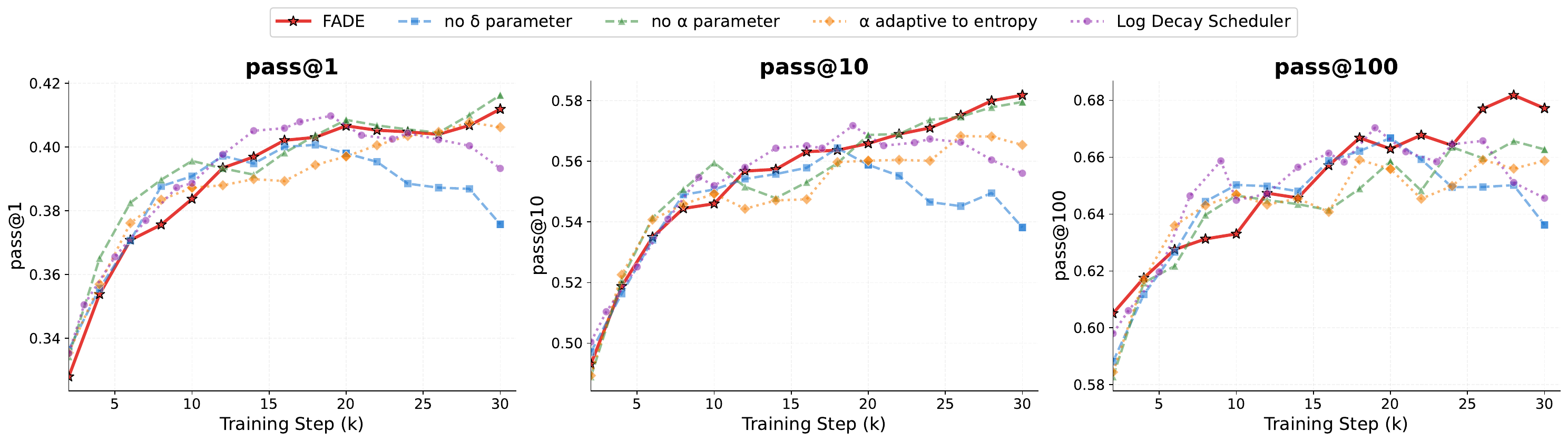}
    \caption{\textbf{Ablation Study on FADE} reveals adaptive $\alpha$ only collapses whereas scheduling both on entropy and ablation $\alpha$ worsens pass@$100$ diversity with Qwen 2.5 7B on LiveCodeBench at 8k reasoning.}
    \label{fig:ablations_fade}
\end{figure}

\end{document}